\documentclass{article}

    \PassOptionsToPackage{numbers, compress}{natbib}
 \usepackage[preprint]{neurips_2026}

\usepackage[utf8]{inputenc} 
\usepackage[T1]{fontenc}    
\usepackage{hyperref}       
\usepackage{url}            
\usepackage{booktabs}       
\usepackage{amsfonts}       
\usepackage{nicefrac}       
\usepackage{microtype}      
\usepackage{xcolor}         

\usepackage{amsmath}
\usepackage{amssymb}
\usepackage{mathtools}
\usepackage{amsthm}
\usepackage{microtype}
\usepackage{graphicx}
\usepackage{subcaption}
\usepackage{hyperref}
\usepackage[export]{adjustbox} 

\theoremstyle{plain}
\DeclareMathOperator*{\argmin}{arg\,min}
\usepackage{algorithm}
\usepackage{algorithmic}
\usepackage{enumitem}
\usepackage{tcolorbox}
\usepackage{multirow}
\usepackage{wrapfig}
\usepackage{caption}
\usepackage{float}

\theoremstyle{plain}
\usepackage{algorithm}
\usepackage{algorithmic}
\usepackage{enumitem}

\usepackage[capitalize,noabbrev,nameinlink]{cleveref}
\crefname{lemma}{lemma}{lemmas}
\Crefname{lemma}{Lemma}{Lemmas}

\crefname{theorem}{theorem}{theorems}
\Crefname{theorem}{Theorem}{Theorems}

\crefname{table}{table}{tables}
\Crefname{table}{Table}{Tables}

\theoremstyle{plain}
\newtheorem{theorem}{Theorem}[section]
\newtheorem{proposition}[theorem]{Proposition}
\newtheorem{lemma}[theorem]{Lemma}

\theoremstyle{definition}
\newtheorem{definition}[theorem]{Definition}
\newtheorem{assumption}[theorem]{Assumption}
\theoremstyle{remark}

\newtheorem{problem}{Problem}

\title{R-GTD: A Geometric Analysis of Gradient Temporal-Difference Learning in Singular Regimes}

\author{%
  Hyunjun Na \quad Donghwan Lee\thanks{Corresponding author.} \\
  School of Electrical Engineering, KAIST\\
  291 Daehak-ro, Yuseong-gu, Daejeon 34141, Korea\\
  \texttt{\{nhjun,donghwan\}@kaist.ac.kr} \\
}

\begin{document}

\maketitle

\begin{abstract}
Gradient temporal-difference (GTD) learning algorithms are widely used for off-policy policy evaluation with function approximation. 
However, existing convergence analyses rely on the restrictive assumption that the so-called feature interaction matrix (FIM) is nonsingular. In practice, the FIM can become singular and leads to instability or degraded performance. While some prior works have applied regularization to relax the nonsingularity assumption, their theoretical guarantees inevitably rely on other restrictive conditions. In this paper, we propose a regularized optimization objective by reformulating the mean-square projected Bellman error minimization. This formulation naturally yields a regularized GTD algorithms, referred to as R-GTD, which guarantees convergence to a unique solution even when the FIM is singular. We conduct a geometric analysis to establish theoretical convergence guarantees and explicit error bounds for the proposed method, and validate its effectiveness through empirical experiments.
\end{abstract}

\section{Introduction}
Temporal-difference (TD) learning~\citep{sutton1988learning} constitutes a foundational approach for policy evaluation in reinforcement learning (RL)~\citep{sutton1998reinforcement}. Despite its effectiveness, it is now well understood that instability and divergence may arise when TD learning is combined with function approximation and bootstrapping in an off-policy setting.
This combination, commonly referred to as the \emph{deadly triad}, has been identified as a fundamental source of instability in RL \citep{tsitsiklis1996analysis}.
To resolve the convergence issue, gradient temporal-difference (GTD) was introduced as an alternative to TD learning~\citep{sutton2008convergent}. Unlike TD learning, which relies on semi-gradient updates, GTD is derived from the true stochastic gradient of a well-defined objective function known as the mean-square Bellman error (MSBE), whose minimization is equivalent to solving the policy evaluation problem by enforcing the Bellman equation in a least-squares sense. Unlike the TD learning, GTD is known to guarantee convergence under off-policy training with linear function approximation. 
However, it still faces limitations with linear function approximation. 
Because the Bellman operator typically maps a value function outside the class of functions representable by the linear function approximation. 
As a result, even the optimal solution of the MSBE may exhibit a nonzero residual Bellman error. To address this issue, GTD2~\citep{sutton2009fast} considers the projected Bellman equation (PBE), where the Bellman update is projected onto the feature space. Specifically, GTD2 minimizes the mean-square projected Bellman error (MSPBE), which measures the squared residual of the PBE.
By enforcing the Bellman equation within the feature space, the projected Bellman error can be driven to zero, and it allows GTD2 to obtain a more accurate approximation of the value function that can be represented by the chosen linear features.

From an optimization perspective, the MSPBE objective admits an equivalent
saddle-point formulation, under which GTD2 can be interpreted as solving a
convex--concave saddle-point problem~\citep{macua2014distributed}.
This perspective allows the use of analytical tools such as Fenchel duality~\citep{boyd2004convex} and Lagrangian methods~\citep{dai2018sbeed}. Recent work in \citet{lee2022new} has extended this perspective by proposing new saddle-point formulations and further GTD2 variants that exploit these structural insights. 
Alongside these perspectives, \citet{ghiassian2020gradient} introduces regularization to improve robustness, and other general regularization techniques have been proposed to stabilize various GTD objectives~\citep{yu2017convergence} or to ensure stability of the critic in off-policy control settings~\citep{zhang2020provably}.
However, existing GTD methods in the literature typically require a nonsingularity assumption on a certain matrix which is associated with feature interaction. Under this assumption, the solution to the PBE is well defined and unique, which in turn ensures the existence of a unique minimizer of the MSPBE. In this paper, for notational convenience, 
we refer to this matrix as the
\emph{feature interaction matrix} (FIM). While some prior works have applied regularization to handle cases where FIM is singular, their theoretical guarantees inevitably rely on other highly restrictive conditions.
For instance, they either  restrict the optimal solution to the origin~\citep{yu2017convergence}, or resort to the global non-singularity assumption to derive error bounds~\citep{zhang2020provably}.
Moreover, mere convergence leaves the geometric link to the original GTD2 solution completely unexplored. To further elucidate this theoretical gap, a detailed mathematical comparison with these regularization-based methods is provided in \Cref{sec:comparison_tdrc}.

In this paper, we propose regularized gradient temporal-difference (R-GTD), which guarantees convergence to a unique solution even when the FIM is singular without introducing additional conditions. Furthermore, we provide a geometric analysis that explicitly characterizes the limiting behavior of the R-GTD solution within the GTD2 affine solution set, which deepens the theoretical understanding of GTD.

Finally, the main contributions are summarized as follows:
\begin{enumerate}
\item[(a)] We propose R-GTD, which removes the nonsingularity assumption on the FIM required for the convergence analysis. While prior regularized methods handle singular settings by introducing other restrictive conditions, R-GTD establishes rigorous convergence guarantees without requiring any additional assumptions.

\item[(b)] We study the theoretical properties of R-GTD from two complementary perspectives. First, we analyze the convergence of R-GTD toward the saddle point of the associated convex--concave formulation. Second, we provide a geometric analysis of the singular regime. We establish explicit error bounds between the R-GTD solution and the true projected solution in both nonsingular and singular settings, which precisely characterize the limiting behavior of the R-GTD solution within the null space of the FIM.

\item[(c)] We empirically validate the proposed algorithm in a representative setting where the FIM is strictly singular. R-GTD converges consistently to the unique saddle-point solution with more stable behavior. Additional experiments in \Cref{app:additional-exp} extend this evaluation to diverse scenarios.
\end{enumerate}

\section{Preliminaries}
\subsection{Markov decision process}
A Markov decision process (MDP) is defined as a tuple 
${\mathcal M} = ({\mathcal S}, {\mathcal A}, P, r, \gamma)$, 
where ${\mathcal S}$ and ${\mathcal A}$ denote the finite state and action spaces, respectively. 
The transition probability $P(s'|s,a)$ specifies the probability of moving from state $s$ to $s'$ 
when action $a$ is taken, and 
$r : {\mathcal S} \times {\mathcal A} \times {\mathcal S} \to {\mathbb R}$ 
is the reward function. 
The discount factor $\gamma \in (0,1)$ determines the relative importance of future rewards. When an agent in state $s$ selects action $a$, it transits to the next state $s'$ 
with probability $P(s'|s,a)$ and receives an immediate reward $r(s,a,s')$. 
A stochastic policy $\pi : {\mathcal S} \times {\mathcal A} \to [0,1]$ 
specifies the probability $\pi(a|s)$ of selecting action $a$ at state $s$. 
Let $P^\pi$ denote the state transition matrix induced by policy $\pi$ and
$D$ denote a diagonal matrix with positive diagonal elements
$d(s)$ for $s \in \mathcal{S}$. Here, $d$ can be any state distribution satisfying $d(s) > 0$ for all
$s \in \mathcal{S}$.
As a particular case, $d$ may be chosen as the stationary state distribution
$d^{\beta}$ induced by a behavior policy $\beta$.
The expected reward under policy $\pi$ at state $s$ is denoted by $R^\pi(s)$. The infinite-horizon discounted value function associated with $\pi$ is defined as
\[
J^\pi(s) := {\mathbb E}\!\left[
\sum_{k=0}^{\infty} \gamma^k r(s_k, a_k, s_{k+1})
\,\middle|\, s_0 = s,\; \pi
\right],
\]
where ${\mathbb E}[\cdot]$ denotes the expectation with respect to the 
state--action trajectories generated by $\pi$. Given a set of basis (feature) functions 
$\phi_1, \ldots, \phi_q : {\mathcal S} \to {\mathbb R}$, 
we define the \emph{feature matrix} 
$\Phi \in {\mathbb R}^{|{\mathcal S}| \times q}$ 
whose $s$-th row is given by 
$\phi(s) := [\phi_1(s)\; \cdots\; \phi_q(s)]$. Here, $q$ is the number of
feature functions and $|{\cal S}|$ denote the cardinality of the state. Throughout this paper, we assume that $\Phi$ has full column rank. Under this setting, the \emph{policy evaluation problem} aims to estimate the value function 
$J^\pi$ corresponding to a fixed policy $\pi$.

\subsection{Saddle-point problem}\label{sec:saddle-point-problem}

We briefly review the saddle-point problem
\citep{nedic2009subgradient,qu2018exponential}.
Let $L:{\mathbb R}^n \times {\mathbb R}^n \to {\mathbb R}$ be a convex--concave function.
We begin with the standard definition of a saddle point.
\begin{definition}[Saddle point]
A pair $(\theta^*, \lambda^*) \in {\mathbb R}^n \times {\mathbb R}^n$
is a saddle point of $L$ if
\[
L(\theta^*, \lambda) \le L(\theta^*, \lambda^*) \le L(\theta, \lambda^*)
\quad \forall (\theta, \lambda).
\]
\end{definition}

When $L$ is convex in $\theta$ and concave in $\lambda$, the saddle-point
condition is equivalent to the following min--max problem:
\begin{problem}[Min--max problem]\label{problem:saddle-point}
\[
\min_{\theta \in {\mathbb R}^n} \max_{\lambda \in {\mathbb R}^n} L(\theta, \lambda)
=
\max_{\lambda \in {\mathbb R}^n} \min_{\theta \in {\mathbb R}^n} L(\theta, \lambda).
\]
\end{problem}
This equivalence requires the convex--concave structure of $L$ together with mild regularity conditions such as coercivity
\citep{sion1958general,rockafellar2015convex}.
Under these conditions, a pair $(\theta^*, \lambda^*)$ is a saddle point if and
only if it satisfies the first-order stationarity conditions
\begin{equation}\label{eq:stationary-condition}
\nabla_\theta L(\theta^*, \lambda^*) = 0, \; 
\nabla_\lambda L(\theta^*, \lambda^*) = 0.
\end{equation}
This characterization motivates gradient-based methods for finding such saddle points. 
A standard approach is the primal--dual gradient dynamic (PDGD) method \citep{nedic2009subgradient}, whose discrete-time iteration is given as
\begin{align*}
\theta_{k+1} = \theta_k - \alpha_k \nabla_\theta L(\theta_k, \lambda_k),\quad
\lambda_{k+1} = \lambda_k + \alpha_k \nabla_\lambda L(\theta_k, \lambda_k).
\end{align*}
Moreover, we can also consider a continuous-time counterpart of PDGD given as
\begin{align}\label{continuous-pdgd}
\dot{\theta}_t = -\nabla_\theta L(\theta_t, \lambda_t),\;
\dot{\lambda}_t = \nabla_\lambda L(\theta_t, \lambda_t),
\end{align}
which converges to a saddle point under standard assumptions
\citep{nedic2009subgradient,qu2018exponential}.
When only stochastic approximations of the gradients are available, the corresponding stochastic formulation is given as follows:
\begin{align*}
\theta_{k+1} = \theta_k - \alpha_k \nabla_\theta \bigl(L(\theta_k, \lambda_k) + v_k\bigr),\quad
\lambda_{k+1} = \lambda_k + \alpha_k \nabla_\lambda \bigl(L(\theta_k, \lambda_k) + w_k\bigr),
\end{align*}
where $(v_k, w_k) \in {\mathbb R}^n \times {\mathbb R}^n$ are i.i.d.\ zero-mean noise terms. We refer to this version as stochastic
PDGD. The stochastic PDGD is also known to converge to a saddle point in a probabilistic sense~\citep{chen2016stochastic,wang2016online}.

\subsection{Review of GTD2}
In this section, we provide a brief overview of the GTD2 algorithm
introduced in~\citet{sutton2009fast} which tries to solve the policy evaluation problem.
In the policy evaluation problem with linear function approximation, the Bellman equation $\Phi \theta = R^{\pi} + \gamma P^{\pi} \Phi \theta$ generally does not admit a solution within the feature subspace.
To address this issue, GTD2 considers the projected Bellman equation (PBE):
\begin{equation}\label{PBE}
\Phi \theta = \Pi_{\mathcal{R}(\Phi)} \bigl(R^{\pi} + \gamma P^{\pi} \Phi \theta\bigr),
\end{equation}
which seeks a solution within the range space of $\Phi$. Here, $\Pi_{\mathcal{R}(\Phi)}$ denotes the projection onto the range space of
$\Phi$, denoted by $\mathcal{R}(\Phi)$: $\Pi_{\mathcal{R}(\Phi)}(x):=\argmin_{x'\in \mathcal{R}(\Phi)}
\|x-x'\|_{D^{\beta}}^2$. 
GTD2 formulates this problem by minimizing the mean-square projected Bellman error (MSPBE), whose minimizer coincides with the solution of the PBE. 
It then seeks this solution in an incremental and model-free manner via stochastic gradient updates:
\begin{problem}\label{problem:2}
Solve for $\theta \in {\mathbb R}^q$ the optimization
\begin{align*}
\min_{\theta\in {\mathbb R}^q} {\rm MSPBE}(\theta)
:= \frac{1}{2}\|
\Pi_{\operatorname{\cal{R}}(\Phi)} (R^{\pi} + \gamma P^{\pi} \Phi \theta)-\Phi \theta \|_{D^{\beta}}^2.
\end{align*}
\end{problem}

Here, $\|x\|_{D^\beta}$ is defined as $\sqrt{x^\top D^\beta x}$ for any positive-definite matrix $D^\beta$. Note that minimizing the objective means minimizing the error of the PBE in \eqref{PBE}
with respect to $ \|\cdot\|_{D^\beta}$. Moreover, in the objective of~\Cref{problem:2}, $d^{\beta}$ depends on the behavior policy, $\beta$, while $P^{\pi}$ and $R^{\pi}$ depend on the target policy, $\pi$, that we want to evaluate. This structure allows us to obtain an off-policy learning algorithm through the importance sampling~\citep{precup2001off} or sub-sampling techniques~\citep{sutton2008convergent}.

Although GTD2 is a promising algorithm with convergence guarantees
under off-policy training, its theoretical analysis relies on the following assumption:
\begin{assumption}\label{assumption:2}
$\Phi ^\top D^{\beta} (\gamma P^\pi   - I)\Phi $ is nonsingular, where $I$ denotes the identity matrix with an appropriate dimension.
\end{assumption}
Note that~\Cref{assumption:2} is common in the literature, and is adopted in~\citet{sutton2008convergent,sutton2009fast,ghiassian2020gradient} for convergence of GTD methods. 
In this paper, we refer to $\Phi^\top D^{\beta} (\gamma P^\pi - I)\Phi$ as the \emph{feature interaction matrix} (FIM), and we say the problem is \emph{ill-conditioned} when the FIM is singular or nearly singular. Under~\Cref{assumption:2}, the MSPBE minimization in~\Cref{problem:2}
has a unique solution. In particular, the GTD2 solution is given by 
\begin{equation}\label{eqn:gtd2_solution}
\theta_{\mathrm{GTD2}}
= -\bigl(\Phi^\top D^{\beta}(\gamma P^{\pi} - I)\Phi\bigr)^{-1}
\Phi^\top D^{\beta} R^{\pi},
\end{equation} 
as shown in~\citet{macua2014distributed}, which coincides with the unique solution of the PBE under \Cref{assumption:2}. 
More recently, to provide a deeper optimization-theoretic perspective,
\citet{lee2022new} showed that the MSPBE minimization 
in \Cref{problem:2} can be equivalently expressed as the following
constrained optimization problem: 
\begin{problem}\label{problem:3}
Solve for $\theta  \in {\mathbb R}^q$ the optimization
\begin{align*}
&\min_{\theta  \in {\mathbb R}^q } 0\quad {\rm s.t.}\quad 0 = \Phi^\top D^{\beta} (R^\pi   + \gamma P^\pi  \Phi \theta  - \Phi \theta ).\label{eq:2}
\end{align*}
\end{problem}
It can be easily proved that the equality constraint in \Cref{problem:3} is equivalent to the PBE in \eqref{PBE}. Specifically, this constraint can be rewritten as $\Phi^\top D^\beta (I - \gamma P^\pi) \Phi \theta = \Phi^\top D^\beta R^\pi$, which leads to the same solution as the GTD2 algorithm in \eqref{eqn:gtd2_solution}. This constrained formulation not only characterizes the PBE solution but also naturally admits an equivalent saddle-point representation.




\section{Regularized GTD}
Before diving into the derivation of our algorithm, we provide a brief roadmap of the optimization problems introduced in this section and the next. We first propose a regularized optimization objective (\Cref{problem:4}), which establishes the foundation for both the proposed algorithm and its saddle-point convergence analysis. To derive the closed-form solution of the proposed algorithm, we then recast this objective into a min-max saddle-point formulation (\Cref{problem:rgtd-lagrangian}). Finally, to facilitate the analysis of the error bounds, we present an equivalent unconstrained formulation (\Cref{problem:6}). 

Following this framework, we reformulate the policy evaluation problem by extending the constrained optimization in \Cref{problem:3}. First, we observe that adding a quadratic term $\frac{1}{2}\theta^\top \Phi^\top D^\beta \Phi \theta$ to the objective of \Cref{problem:3} does not shift the optimal solution $\theta_{\text{GTD2}}$ as long as the FIM is nonsingular (\Cref{assumption:2}). This term serves as a base for regularization without altering the solution. 
However, when the FIM is singular, the added quadratic term
selects a particular solution among them, thereby
altering the optimal solution.
Next, to ensure the problem remains well-defined even when FIM is singular, we introduce an auxiliary variable $w \in \mathbb{R}^q$ into the constraint. This $w$ acts as a slack variable that allows for a small error in the PBE.
Finally, to prevent this error from becoming large, we incorporate a penalty term $\frac{c}{2}w^\top \Phi^\top D^\beta \Phi w$ into the objective, where $c > 0$ is a regularization coefficient. This leads to the following regularized optimization problem:
\begin{problem}\label{problem:4}
Solve for $\theta,w  \in {\mathbb R}^q $ the optimization
\begin{equation}\label{eq:problem-5}
\begin{aligned}
&\min_{\theta,w  \in {\mathbb R}^q } \frac{c}{2}w^\top \Phi ^\top D^{\beta} \Phi w+\frac{1}{2}\theta ^\top \Phi ^\top D^{\beta} \Phi \theta\\
&{\rm s.t.}\quad 0 = \Phi ^\top D^{\beta} (R^\pi   + \gamma P^\pi  \Phi \theta  - \Phi \theta + \Phi w).
\end{aligned}
\end{equation}
\end{problem}
Here, we allow a weighted error term 
$\Phi w$ in the constraint compared to \Cref{problem:3}, which effectively introduces a small error in the PBE in \eqref{PBE}.
The variable $w$ represents an auxiliary variable capturing the error of the
PBE, and $c>0$ is a regularization parameter that
controls its penalty. Note that the regularization parameter $c$ controls the trade-off between the constraint satisfaction and the objective penalty. A detailed practical guidelines for selecting $c$ in practice are provided in \Cref{sec:sensitivity_guideline}.

Based on this formulation, we derive the R-GTD algorithm in this paper. The first step in this derivation is to recast~\Cref{problem:4} as a min–max saddle-point problem in \Cref{problem:saddle-point}, for which we now introduce the corresponding Lagrangian function 
\begin{align}\label{eqn:r-gtd-lagrangian}
L(\theta , w, \lambda ) &= \frac{c}{2}w ^\top \Phi ^\top D^{\beta} \Phi w  + \frac{1}{2}\theta ^\top \Phi ^\top D^{\beta} \Phi \theta  + \lambda ^\top \Phi ^\top D^{\beta} (R^\pi   + \gamma P^\pi  \Phi \theta  - \Phi \theta +\Phi w),
\end{align}
where we can prove that $L(\theta,w,\lambda)$ is strongly convex in the primal variable 
$\theta$ and $w$, concave in the dual variable $\lambda$, and coercive due to the positive definite weighting matrix. Therefore, by \Cref{lem:saddle-existence-simple} of Appendix, the problem satisfies the structural conditions of the min--max formulation in \Cref{problem:saddle-point} and admits a saddle point. 
Moreover, due to the strong convexity of $L(\theta,w,\lambda)$ in the primal
variables $(\theta,w)$ and its concavity in the dual variable $\lambda$,
the min--max problem admits a unique
saddle-point solution.
The corresponding min-max saddle-point problem is summarized below for convenience.
\begin{problem}\label{problem:rgtd-lagrangian}
Solve for $\theta, w, \lambda \in \mathbb{R}^q$ the optimization
\begin{align*}
\min_{\theta, w  \in {\mathbb R}^q } \max_{\lambda  \in {\mathbb R}^q } L(\theta , w, \lambda ): &=  \frac{c}{2}w ^\top \Phi ^\top D^{\beta} \Phi w  + \frac{1}{2}\theta ^\top \Phi ^\top D^{\beta} \Phi \theta  \nonumber+ \lambda ^\top \Phi ^\top D^{\beta} \\&\quad \times (R^\pi  + \gamma P^\pi  \Phi \theta  - \Phi \theta\nonumber+\Phi w).
\end{align*}
\end{problem}

Solving the stationary conditions in \eqref{eq:stationary-condition} yields a closed-form characterization
of the optimal solution $(\theta_{\mathrm{RGTD}}, \lambda_{\mathrm{RGTD}}, w_{\mathrm{RGTD}})$
to the min-max problem in \Cref{problem:rgtd-lagrangian}.
For clarity, the explicit expressions of the optimal solution are summarized below.
\begin{proposition}\label{prop:1}
Let $(\theta_{\mathrm{RGTD}}, \lambda_{\mathrm{RGTD}}, w_{\mathrm{RGTD}})$
denote the optimal solution to the min-max problem in \Cref{problem:rgtd-lagrangian}.
Then the optimal solution admits the following closed-form expressions:
\begin{align}
\theta_{\mathrm{RGTD}}
&=
\underbrace{
-\bigl(\Phi^\top D^{\beta}(\gamma P^\pi - I)\Phi\bigr)^{-1}
\Phi^\top D^{\beta} R^\pi
}_{\theta_{GTD2}}
\underbrace{
-\bigl(\Phi^\top D^{\beta}(\gamma P^\pi - I)\Phi\bigr)^{-1}
\Phi^\top D^{\beta}\Phi \, w_{\mathrm{RGTD}}
}_{\text{error}},
\label{eq:theta-rgtd}
\\
\lambda_{\mathrm{RGTD}}
&=
-\bigl(\Phi^\top D^{\beta}(\gamma P^\pi - I)\Phi\bigr)^{-\top}
\Phi^\top D^{\beta}\Phi \, \theta_{\mathrm{RGTD}},
\label{eq:lambda-rgtd}
\\
w_{\mathrm{RGTD}}
&=
-\frac{1}{c}\,\lambda_{\mathrm{RGTD}}.
\label{eq:w-rgtd}
\end{align}
\end{proposition}
The proof of \Cref{prop:1} is given in \Cref{app:prop-1}. 
The characterization in \Cref{prop:1} reveals that the R-GTD solution
$\theta_{\mathrm{RGTD}}$ does not exactly coincide with the GTD2 solution
$\theta_{\mathrm{GTD2}}$ for finite values of the parameter $c$.
In particular, compared to $\theta_{\mathrm{GTD2}}$ in \eqref{eqn:gtd2_solution}, the $\theta_{\mathrm{RGTD}}$
contains an additional bias term
$-(\Phi^\top D^{\beta}(\gamma P^\pi - I)\Phi)^{-1}
\Phi^\top D^{\beta}\Phi\, w_{\mathrm{RGTD}},
$ which originates from the auxiliary variable $w_{\mathrm{RGTD}}$
introduced in \eqref{eq:problem-5}.
Moreover, by direct manipulations of \eqref{eq:theta-rgtd}--\eqref{eq:w-rgtd}, one can easily prove that $\theta_{\mathrm{RGTD}}$ is explicitly given by the following equation:
\begin{equation}\label{eq:rgtd-closed-form}
\begin{aligned}
\theta_{\mathrm{RGTD}}
&=
-\Bigl(
(\Phi^\top D^{\beta}(\gamma P^\pi - I)\Phi)^\top
(\Phi^\top D^{\beta}\Phi)^{-1}
(\Phi^\top D^{\beta}(\gamma P^\pi - I)\Phi)
+ \tfrac{1}{c}\,\Phi^\top D^{\beta}\Phi
\Bigr)^{-1}
\\
&\quad \times
(\Phi^\top D^{\beta}(\gamma P^\pi - I)\Phi)^\top
(\Phi^\top D^{\beta}\Phi)^{-1}
\Phi^\top D^{\beta} R^\pi .
\end{aligned}
\end{equation}
We can see that as $c \to \infty$, the solution \eqref{eq:rgtd-closed-form} converges to the GTD2 solution in \eqref{eqn:gtd2_solution} provided that the FIM, $\Phi^\top D^{\beta}(\gamma P^\pi - I)\Phi$, is nonsingular. 
When FIM is singular, then still the term inside $(\Phi^\top D^{\beta}(\gamma P^\pi - I)\Phi)
+ \tfrac{1}{c}\,\Phi^\top D^{\beta}\Phi
)^{-1}$ is nonsingular. Therefore, the R-GTD solutions is still well defined in this case. Consequently, under \Cref{assumption:2},
letting $c \to \infty$ eliminates the effect of the auxiliary variable,
yielding
$
\theta_{\mathrm{RGTD}} \rightarrow\theta_{\mathrm{GTD2}},
w_{\mathrm{RGTD}} \to 0,
$ and $
\lambda_{\mathrm{RGTD}} \to 0.
$
 
Having characterized the optimal solution and its properties, we now derive the practical update rules for R-GTD. As discussed earlier, a standard approach for solving such convex--concave
saddle-point problems is to apply the PDGD.
Accordingly, we first consider the continuous-time PDGD in \eqref{continuous-pdgd} associated with the
Lagrangian $L(\theta,w,\lambda)$ in \eqref{eqn:r-gtd-lagrangian}, given by
\begin{align}\label{eq:continuous-rgtd}
\dot{\theta}_t 
    &= -\nabla_\theta L(\theta_t, w_t, \lambda_t)\nonumber
     = -\Phi ^\top D^{\beta}\Phi\, \theta_t 
       - \big((\gamma P^\pi\Phi) ^\top D^{\beta}\Phi 
         - \Phi ^\top D^{\beta}\Phi\big)\lambda_t, \nonumber\\[2mm]
\dot{w}_t 
    &= -\nabla_w L(\theta_t, w_t, \lambda_t)
     = -c\, \Phi ^\top D^{\beta}\Phi\, w_t 
       - \Phi ^\top D^{\beta}\Phi\, \lambda_t, \nonumber\\[2mm]
\dot{\lambda}_t 
    &= \nabla_\lambda L(\theta_t, w_t, \lambda_t)
     = \Phi ^\top D^{\beta}
       (R^\pi + \gamma P^\pi \Phi \theta_t - \Phi \theta_t
       + \Phi w_t).
\end{align}
Its discrete-time counterpart (by Euler discretization) is 
\begin{align}\label{eq:discrete-rgtd}
\theta _{k+1} =& \theta _k +\alpha_k(-\Phi ^\top D^{\beta}\Phi \theta_{k}-((\gamma P^\pi\Phi) ^\top D^{\beta}\Phi\nonumber-\Phi ^\top D^{\beta}\Phi)\lambda _k),\nonumber\\
w_{k+1} =& w_k + \alpha_k(-c \Phi ^\top D^{\beta}\Phi w_k -\Phi ^\top D^{\beta}\Phi\lambda_k),\nonumber \\ \lambda_{k+1} =& \lambda_k + \alpha_k \Phi ^\top D^{\beta}(R^\pi+\gamma P^\pi  \Phi \theta _k- \Phi \theta _k +\Phi w_k),
\end{align}
where the scalar $\alpha_k>0$ denotes the step size at iteration $k$ and $e_s \in \mathbb{R}^{|\mathcal{S}|}$ denotes the standard basis vector associated with
state $s$, whose $s$-th entry is one and all other entries are zero. 
Considering $s_k \,\sim\, d^{\beta}
$, $a_k \,\sim\, \pi(\cdot\,\vert\, s_k)$, and
$s_k' \,\sim\, P(\cdot\,\vert\, s_k, a_k)$,
The corresponding stochastic PDGD can then be obtained as follows:
\begin{align*}
\theta _{k + 1}  =& \theta _k  + \alpha _k 
(-\Phi ^\top e_{s_k}e_{s_k}^\top\Phi \theta_{k}-((\gamma e_{s_k}e_{s_k'}\Phi)^\top e_{s_k}e_{s_k}^\top\Phi-\Phi ^\top e_{s_k}e_{s_k}^\top\Phi )\lambda_k),\\w_{k+1}=&w_k+\alpha_k(-c \Phi ^\top e_{s_k}e_{s_k}^\top\Phi w_k - \Phi ^\top e_{s_k}e_{s_k}^\top\Phi\lambda_k ),\\
\lambda _{k + 1}  =& \lambda _k  + \alpha _k \Phi ^\top e_{s_k } e_{s_k }^\top (e_{s_k } r(s_k ,a_k ,s_k ')+ \gamma e_{s_k } e_{s_k '}^\top \Phi \theta _k  - \Phi \theta _k + \Phi w _k ),
\end{align*}
where $s_k$ is the state sampled at iteration $k$, and $s_k'$ is the next state observed after taking an action $a_k$ according to the behavior policy $\beta$.
The resulting stochastic PDGD updates constitute the R-GTD algorithm,
summarized in \Cref{alg:rgtd-alg}, which incorporates importance sampling
for off-policy data~\citep{precup2001off}.
\begin{algorithm}[tbp]
\caption{R-GTD}
\label{alg:rgtd-alg}
\begin{algorithmic}[1]

\STATE Set the step-size $(\alpha_k)_{k=0}^\infty$ and regularization parameter $c$. Initialize $(\theta _0, w _0, \lambda_0 )$.

\FOR{$k \in \{0,\ldots\}$}

\STATE Observe $s_k \sim d^{\beta}$, $a_k \sim \beta(\cdot|s_k)$, and $s_k'\sim P(\cdot |s_k,a_k)$, $r_{k+1} :=r(s_k,a_k,s_k')$.
\STATE Update parameters according to
\begin{align*}
&\theta _{k + 1}  = \theta _k  + \alpha _k ((\phi_k  - \gamma \rho _k \phi_{k}')(\phi_k ^\top \lambda _k ) - \phi_k (\phi_k ^\top \theta_k )),\\&w_{k+1}=w_k+\alpha_k(-c \phi_k ^\top w_k - \phi_k ^\top \lambda_k)\phi_k,\\
&\lambda_{k+1}=\lambda_k +\alpha_k (\delta_k +\phi_k^\top w_k)\phi_k,
\end{align*}
where $\phi_k:=\phi(s_k),\phi_{k}':=\phi(s_{k}')$, $\rho _k : = \frac{{\pi (a_k |s_k )}}{{\beta (a_k |s_k )}}$, and $\delta_k =\rho _k r_{k+1} +\gamma \rho _k (\phi_{k}')^\top \theta_k -\phi_k^\top \theta_k$.
\ENDFOR
\end{algorithmic}
\label{algo:GTD2}
\end{algorithm}
In particular, it includes the ratio $\rho_k := \pi(a_k \,\vert\, s_k) / \beta(a_k \,\vert\, s_k)
$ to correct for the mismatch between the behavior policy $\beta$, which generates the data, and the target policy $\pi$, whose value function is
being evaluated.
Intuitively, the importance sampling ratio reweights each sample so that,
in expectation, the update direction corresponds to that of on-policy learning
under $\pi$, as commonly done in off-policy temporal-difference methods
\citep{precup2001off,sutton2008convergent}. 

Finally, to facilitate the convergence analysis in the next
section, we introduce an alternative but equivalent unconstrained optimization formulation of \Cref{problem:4} as follows:
\begin{problem}\label{problem:6}
Solve for $\theta \in \mathbb{R}^q$ the optimization
\begin{align}
\min_{\theta \in \mathbb{R}^q}\;
\frac{c}{2}\lVert \Pi_{\mathcal{R}(\Phi)}(R^\pi + \gamma P^\pi \Phi\theta - \Phi\theta)\rVert_{D^{\beta}}^2
+ \frac{1}{2}\lVert \Phi\theta\rVert_{D^{\beta}}^2.
\label{eq:problem-6}
\end{align}
\end{problem} 
An important feature of the above formulation is that it is an unconstrained optimization while  \Cref{problem:4} is a constrained optimization. This alternative formulation will serve as a convenient basis for analyzing the stability and convergence properties of R-GTD. Moreover, 
\Cref{problem:6} can be interpreted as a regularized MSPBE formulation,
where the first term 
is the MSPBE which is the objective of GTD2,
while the second term is a quadratic regularization term.
As $c \to \infty$, \Cref{problem:6} (and equivalently \Cref{problem:4})
yields the same solution as GTD2. The proof is given in \Cref{app:3}.



\section{Convergence of R-GTD}\label{sec:convergence}
In this section, we analyze the theoretical properties of R-GTD in two parts.
First, we establish convergence to a saddle point under the convex--concave
formulation. Second, we derive error bounds relating the R-GTD solution to
the true projected solution.

\subsection{Convergence of R-GTD towards saddle-point}
In this section, we analyze the convergence of R-GTD towards a saddle point by studying its
continuous-time dynamics and the associated discrete-time algorithm. Specifically, we leverage recent results on PDGD
from~\citet{qu2018exponential} to establish the convergence of the continuous-time
PDGD, and subsequently apply the ordinary differential equation (ODE) method of~\citet{borkar2000ode} to prove
the convergence of the stochastic R-GTD iterates. To analyze the convergence of R-GTD, we consider a constrained
optimization in \Cref{problem:4}.
Specifically, we express this constrained optimization in \Cref{problem:4} in the
following general form:
\begin{align}\label{eq:constrained-optimization}
\min_{x \in \mathbb{R}^n} f(x)
\quad \text{s.t.} \quad
Ax = b,
\end{align}
where $f$ is convex and continuously differentiable.
As discussed earlier, a standard approach to solving
\eqref{eq:constrained-optimization} is to introduce the associated
convex--concave Lagrangian function \citep{boyd2004convex},
defined as
\begin{align}
L(x,\lambda) = f(x) + \lambda^\top(Ax - b).
\label{eq:constrained-optimization-lagrangian}
\end{align}
By invoking a key lemma from \citet{qu2018exponential}, whose statement and proof
are provided in \Cref{lemma:1} of \Cref{app:math}, we establish the convergence of
the continuous-time PDGD associated with R-GTD.
When the objective function $f$ is strongly convex and the constraint matrix $A$
has full row rank, \citet{qu2018exponential} shows that the continuous-time PDGD of
\eqref{eq:constrained-optimization-lagrangian},
\[
\dot{x} = -\nabla f(x) - A^\top \lambda, \qquad
\dot{\lambda} = A x - b,
\]
globally converges to the unique saddle point of the associated Lagrangian. However, in the optimization form of GTD2 in \Cref{problem:3}, 
the constraint matrix is given by
$
A = \Phi^\top D^\beta (\gamma P^\pi - I)\Phi,
$
which refers to the FIM introduced earlier.
This matrix satisfies the full row rank condition only under \Cref{assumption:2}. If this assumption is violated, the constraint matrix becomes rank-deficient, and the convergence guarantee provided by \citet{qu2018exponential} no longer applies. When applying the R-GTD formulation, the constraint matrix $A$ in \Cref{problem:4} takes the block form $A=\bigl[\, \text{FIM} \;\; \Phi^\top D^\beta \Phi \,\bigr],
$ where $\Phi^\top D^\beta \Phi \succ 0$.
As a consequence of the additional $\Phi^\top D^\beta \Phi$ term, the resulting
constraint matrix $A$ is always full row rank, regardless of whether
\Cref{assumption:2} holds.
Consequently, the conditions required by \citet{qu2018exponential} are always
satisfied for R-GTD, which ensures global convergence of the associated
continuous-time PDGD.
For completeness, a detailed verification of this rank condition is provided
in \Cref{app-prop:4}. Based on this continuous-time convergence result and \Cref{lemma:Borkar} of Appendix, convergence of~\Cref{alg:rgtd-alg} using ODE approach can be proved as follows:
\begin{theorem}\label{thm:convergence2}
Let us consider~\Cref{alg:rgtd-alg}, and assume that the step-size satisfy~\eqref{eq:step-size-rule}.
Then, $(\theta_k, w_k, \lambda_k)\to ( \theta_{\mathrm{RGTD}}, w_{\mathrm{RGTD}}, \lambda_{\mathrm{RGTD}})$ as $k\to \infty$ with probability one (see \Cref{app:2} for the proof).
\end{theorem}
\subsection{Error bounds between the R-GTD and the true projected solution
}\label{sec:convergence-towards-true}
While we have established convergence to the R-GTD saddle point, it is essential
to quantify its proximity to the true value function $V^\pi$. To this end, we introduce the true projected solution $\theta_*^\pi$, as considered in~\citet{lee2024analysis}, as the vector which satisfies $\Phi \theta_*^\pi = \Pi_{\mathcal{R}(\Phi)} V^\pi,$
where $V^\pi$ admits the representation
$
V^\pi = \sum_{k=0}^\infty \gamma^k (P^\pi)^k R^\pi .
$
This $\theta_*^\pi$ corresponds to the minimizer of the projection error $f(\theta) = \frac{1}{2}\|V^\pi - \Phi\theta\|_{D^\beta}^2$. 
To analyze the error bound, we first introduce the following notations, which will be used repeatedly in the sequel.
\paragraph*{Notation}
Recall that the FIM is defined as
$
\mathrm{FIM} := \Phi^\top D^\beta (\gamma P^\pi - I)\Phi.
$
For notational convenience, we denote this matrix by
$
M := \Phi^\top D^\beta (\gamma P^\pi - I)\Phi,
$
and define 
\[
B := \Phi^\top D^\beta \Phi,\;
b := \Phi^\top D^\beta R^\pi,\;G := M^{^\top} B^{-1} M,
\]
where $B \succ 0$. We further define $K := G^{-1} B G^{-1} M^{^\top} B^{-1} b.
$ 
Using this notation, the MSPBE objective admits the quadratic form $
{\rm MSPBE}(\theta)
= \tfrac12 (M\theta + b)^\top B^{-1} (M\theta + b).
$
Therefore, the corresponding first-order optimality condition can be compactly written as follows:
\begin{equation}
\nabla_\theta {\rm MSPBE}(\theta)
= M^\top B^{-1}(M\theta + b)=0.\label{eq:gtd2-opt}
\end{equation}
The structure of the GTD2 solution implied by this condition is summarized in the following lemma:
\begin{lemma}\label{lem:gtd2-solution}
The GTD2 stationary condition \eqref{eq:gtd2-opt}
admits an affine solution set
\begin{equation}\label{eq:gtd2-set}
\Theta_{\mathrm{GTD2}}
:= \{\theta \in \mathbb{R}^q : M^\top B^{-1}(M\theta + b)=0\}.
\end{equation}\label{eq:gtd2-solution-nonsingular}
When $M$ (FIM) is nonsingular, this set reduces to the singleton
\begin{equation}
\theta_{\mathrm{GTD2}}
= -(M^\top B^{-1}M)^{-1}M^\top B^{-1}b.
\end{equation}
\end{lemma}

To derive the subsequent error bounds, we rewrite the optimality condition
of the R-GTD problem in \eqref{eq:problem-6} as an equivalent linear system.
This reformulation yields
$\bigl(M^\top B^{-1} M + \tfrac{1}{c}B\bigr)\theta_{\mathrm{RGTD}}
= - M^\top B^{-1} b$.
The derivation of this linear system form is provided in
\Cref{lem:5} of Appendix. Consequently, the R-GTD solution admits the closed-form expression
\begin{equation}\label{eq:rgtd-solution}
\theta_{\mathrm{RGTD}}
:= -\left(M^\top B^{-1}M + \tfrac{1}{c}B\right)^{-1} M^\top B^{-1} b .
\end{equation}
Note that the expression in \eqref{eq:rgtd-solution} coincides with the
closed-form solution given in \eqref{eq:rgtd-closed-form}.

We now characterize the asymptotic relationship between $\theta_{\mathrm{RGTD}}$ and $\theta_{\mathrm{GTD2}}$ as the regularization parameter $c$ increases in the following lemma:
\begin{lemma}\label{lem:6}
Let $M,B,b,G,K$ be defined as above, 
and let $\theta_{\mathrm{RGTD}}$ be given by \eqref{eq:rgtd-solution}. Let $\Pi_{\mathcal{N}(G)}
$ denote the orthogonal projection onto the null space of $G$.
If $M$ is singular, then there exists a constant
$
c_0 := \|B - \Pi_{\mathcal{N}(G)}\|_2
$
such that for all $c > c_0$, every
$\theta_{\mathrm{GTD2}}\in\Theta_{\mathrm{GTD2}}$
defined in \eqref{eq:gtd2-set} admits the expansion
\begin{equation}\label{eq:rgtd-singular-expansion}
\theta_{\mathrm{RGTD}}
=
\bigl(\theta_{\mathrm{GTD2}}
      - \Pi_{\mathcal{N}(G)}(\theta_{\mathrm{GTD2}})
 \bigr)
+ O\!\left(\frac{1}{c}\right).
\end{equation}
If $M$ is nonsingular, then there exists a constant
$
c_0 := \|G^{-1}B\|_2
$
such that for all $c > c_0$, the unique GTD2 solution
$\theta_{\mathrm{GTD2}}$ given in
\eqref{eq:gtd2-solution-nonsingular} admits the expansion
\begin{equation}\label{eq:rgtd-nonsingular-expansion}
\theta_{\mathrm{RGTD}}
= \theta_{\mathrm{GTD2}}
+ \frac{1}{c} K
+ O\!\left(\frac{1}{c^2}\right).
\end{equation}
\end{lemma}

\begin{wrapfigure}{r}{0.6\textwidth}
\vspace{-10pt}
\raggedleft
\includegraphics[width=\linewidth, trim=180 100 250 180, clip]{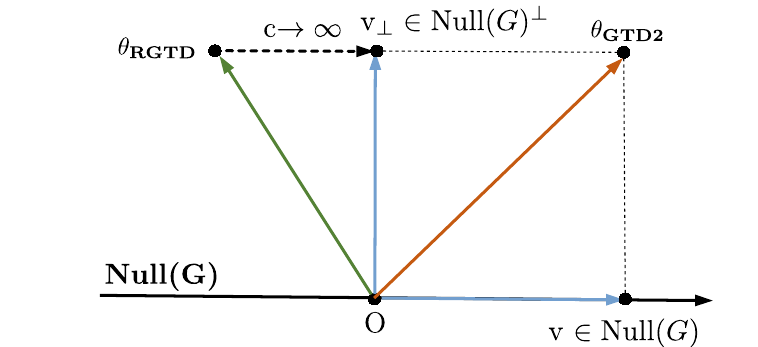}
\captionsetup{width=\linewidth}
\caption{As $c \to \infty$, the R-GTD solution $\theta_{\mathrm{RGTD}}$
converges to the GTD2 solution $\theta_{\mathrm{GTD2}}$.
$\theta_{\mathrm{GTD2}}$ decomposes uniquely into two components:
$v \in \mathrm{Null}(G)$ along the null space of $G$, and 
$v_{\perp} \in \mathrm{Null}(G)^{\perp}$ orthogonal to it.}
\label{fig:nullspace}
\vspace{-10pt}
\end{wrapfigure}

The proof is given in \Cref{app-lem:6}. While the nonsingular case is relatively straightforward as $\theta_{\mathrm{RGTD}}$ converges to a unique point, the singular case involves a solution set $\Theta_{\mathrm{GTD2}}$ that is an affine subspace. 
We therefore focus on the singular case to clarify how the regularization influences the limiting behavior of $\theta_{\mathrm{RGTD}}$ within this solution set. \Cref{fig:nullspace} provides a geometric illustration of the expansion
in~\eqref{eq:rgtd-singular-expansion}.
In particular, as $c \to \infty$, the $O(1/c)$ term vanishes and the right-hand
side of~\eqref{eq:rgtd-singular-expansion} reduces to the orthogonal projection of a GTD2 solution onto $\mathrm{Null}(G)^{\perp}$, as depicted in the figure. 
To build intuition for the singular-case asymptotic behavior in \eqref{eq:rgtd-singular-expansion}, we provide a toy example in \Cref{app:toy}.

Building on these properties, we further quantify the Euclidean distance between $\Phi\theta_{\mathrm{RGTD}}$ and the GTD2 solution set in \Cref{lem:7} of the Appendix. These results lead to our main theorem, which establishes explicit error bounds between the R-GTD solution and the true projected solution.
\begin{theorem}\label{thm:2}

The prediction error satisfies
\begin{align*}
\|\Phi\theta_{\mathrm{RGTD}} - \Phi\theta_*^\pi\|_2
&\le 
\mathrm{dist}(\Phi\theta_*^\pi,\; \Phi\Theta_{\mathrm{GTD2}})
+
\bigl\|
\Pi_{\Phi\Theta_{\mathrm{GTD2}}}(\Phi\theta_{\mathrm{RGTD}})
-
\Pi_{\Phi\Theta_{\mathrm{GTD2}}}(\Phi\theta_*^\pi)
\bigr\|_2
+
O\!\left(\tfrac{1}{c}\right).
\end{align*}
If $M$ is nonsingular, then $\Theta_{\mathrm{GTD2}}=\{\theta_{\mathrm{GTD2}}\}$ 
and the projection operator becomes the identity.
In this case, the bound simplifies to
\begin{align*}
\|\Phi\theta_{\mathrm{RGTD}} - \Phi\theta_*^\pi\|_2
&\le
\|\Phi\theta_{\mathrm{GTD2}} - \Phi\theta_*^\pi\|_2
+ \frac{\|\Phi\|_2\|K\|_2}{c}
+ O\!\left(\tfrac{1}{c^2}\right).
\end{align*}
\end{theorem}
The proof is given in \Cref{app-thm:2}. \Cref{thm:2} highlights the practical advantage of R-GTD over GTD2 by decomposing the error bound into two distinct components. In the \emph{ill-conditioned case}, where \Cref{assumption:2} does not hold, the first term, $dist(\Phi\theta_*^\pi, \Phi\Theta_{GTD2})$, represents the intrinsic approximation error of GTD2 arising from its affine solution set. The second term reflects the stability of the specific representative chosen by R-GTD within this set. Crucially, because the original MSPBE objective lacks strong convexity in the singular setting, the limiting behavior of specific stochastic GTD2 trajectories  cannot be uniquely characterized, which necessitates a set-based theoretical comparison. In contrast, the regularization in R-GTD ensures a strongly convex objective, guaranteeing that all stochastic trajectories reliably converge to a unique point. \Cref{fig:distance_figure} visualizes this singular case decomposition. In the \emph{nonsingular case}, GTD2 already provides the optimal approximation, but R-GTD converges to the same solution as $c \to \infty$. Thus, even in nonsingular settings, R-GTD does not degrade performance: its error is at most an $O(1/c)$ bias away from GTD2. 


\begin{figure*}[t]
\centering
\includegraphics[width=0.7\textwidth, trim=0 110 0 0, clip]{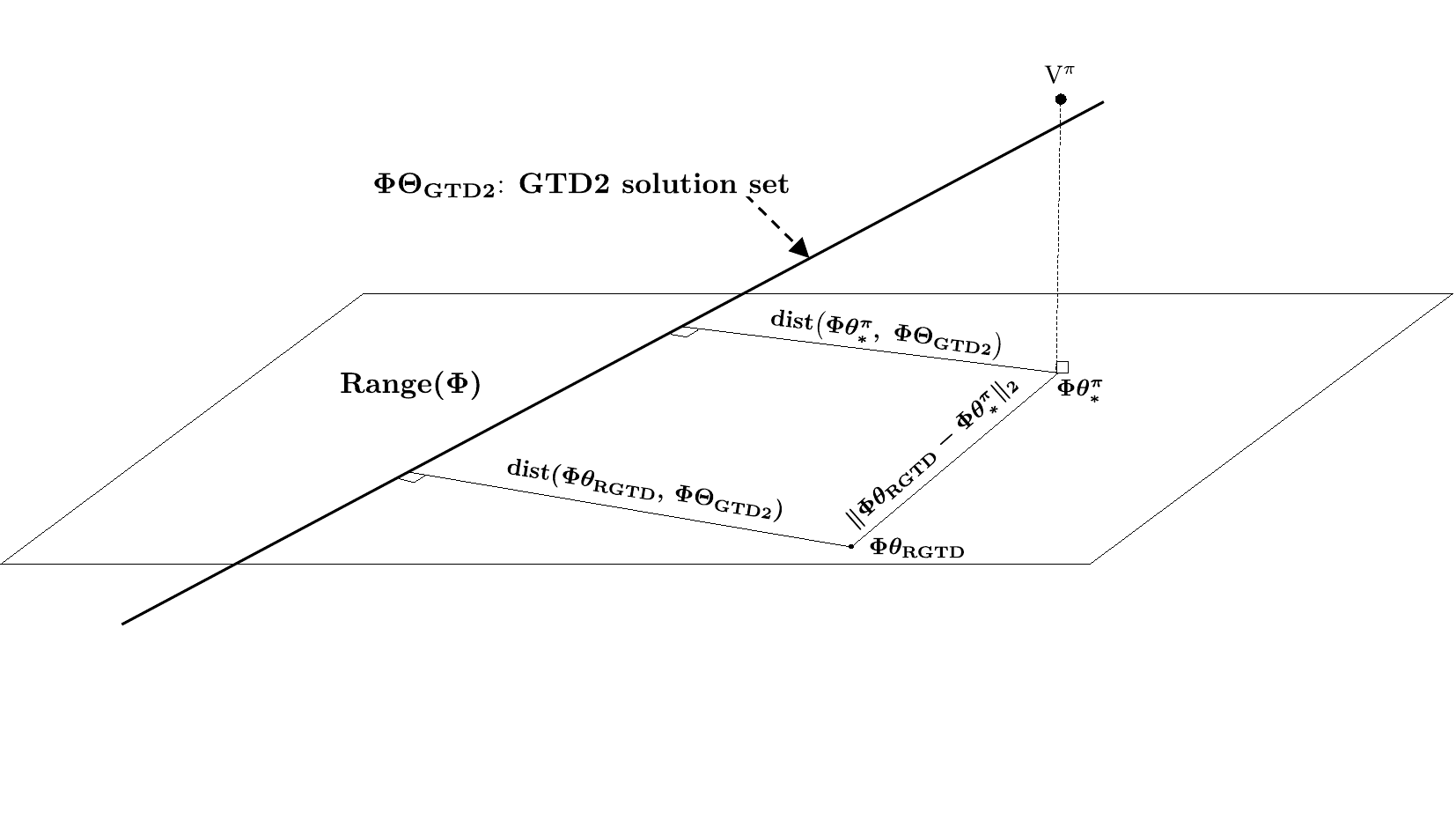}
\caption{Geometric illustration of the singular case.
The error bounds of R-GTD and GTD2 are compared through the projection and distance
terms in \Cref{thm:2}.}
\label{fig:distance_figure}
\end{figure*}

\section{Experiments}

\begin{wrapfigure}{r}{0.4\textwidth}
\vspace{-10pt}
\raggedleft
\includegraphics[width=\linewidth]{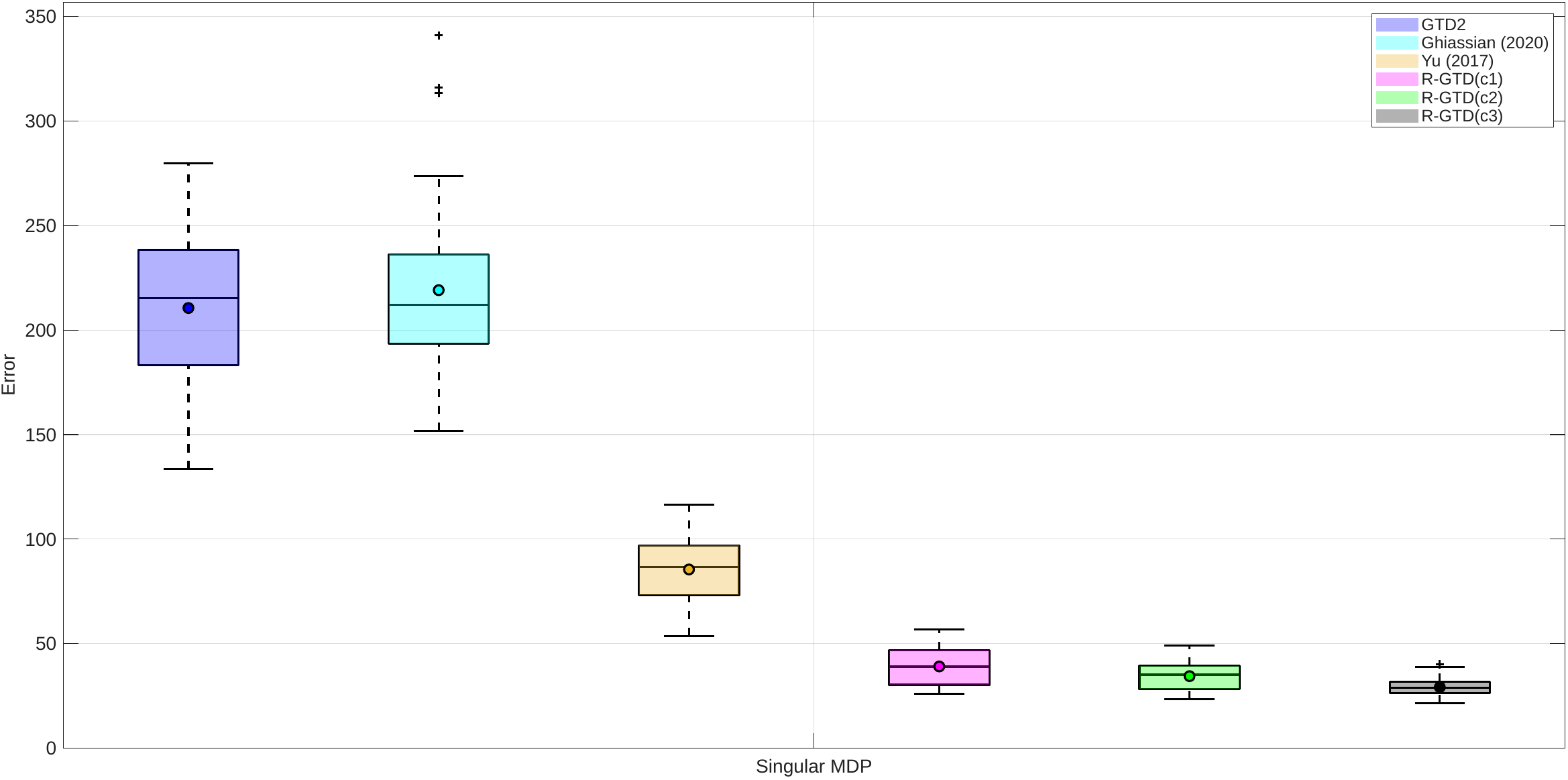}
\captionsetup{width=\linewidth}
\caption{%
    Box-and-whisker representation of $\|\theta_k - \theta_*^\pi\|_2$ 
    over 30 independent runs with step-size 
    $\alpha_k = 1/(k + 30)$.
}
\label{fig:singular-case}
\vspace{-10pt}
\end{wrapfigure}

We consider a representative experiment in a singular setting, to evaluate the stability and robustness of
the proposed R-GTD algorithm.
Specifically, we construct a scenario where the matrix FIM is ill-conditioned, with its smallest singular value on the order of $10^{-13}$.
We initialize $\theta_0$ with a null-space component of FIM. The regularization coefficients are chosen as $c \in \{0.2, 0.4, 1\}$. We compare R-GTD with regularization-based baselines 
\cite{ghiassian2020gradient, yu2017convergence}.
Note that while \citet{zhang2020provably} also employs a similar regularization, it is fundamentally an off-policy control algorithm. Therefore, we restrict our empirical comparison to \citet{ghiassian2020gradient, yu2017convergence}, which directly target the same problem setting. Each experiment is repeated $30$ times, and the distribution of
$\|\theta_k - \theta_*^\pi\|_2$ is summarized using box-and-whisker plots. Detailed experimental settings are provided in \Cref{sec:robustness_initialization}.
As shown in \Cref{fig:singular-case}, R-GTD suppresses the null-space component
and converges stably toward $\theta_*^\pi$. Additional experiment results are provided in \Cref{app:additional-exp}.

\section{Conclusion and limitations}\label{sec:conclusion}
This paper introduced R-GTD, a regularized variant of GTD2 derived from a
convex--concave saddle-point formulation.
The proposed approach relaxes the convergence assumptions of GTD2 while preserving the core theoretical structure of GTD2. Beyond convergence, we further provided a geometric analysis that characterizes the limiting behavior of the R-GTD solution within the GTD2 solution set in singular regimes. As a result, the method achieves improved stability in ill-conditioned settings.
We proved asymptotic convergence to the saddle point and established error bounds
with respect to the true projected solution. 
However, the current analysis is limited to linear function approximation and asymptotic convergence. 
Extending the framework to nonlinear function approximation and establishing finite-time guarantees remain important directions for future work.


\bibliographystyle{plainnat}
\bibliography{example_paper}

\newpage
\appendix

\section{Mathematical preliminaries}\label{app:math}

To facilitate the convergence analysis of the proposed regularized GTD (R-GTD) algorithm, this appendix summarizes several mathematical preliminaries that are used throughout the main text.
We collect the basic notions from nonlinear system theory, the ordinary differential equation (ODE)-based stochastic approximation framework, and key results on primal–dual gradient dynamics (PDGD).
These tools provide the technical foundation for the stability proofs, ODE limiting arguments, and saddle-point convergence results appearing in \Cref{sec:convergence}.
For completeness, we restate the assumptions and lemmas in a self-contained manner so that readers can verify the theoretical arguments without referring to external sources.

\subsection{Basics of nonlinear system theory}
We briefly review basic concepts from nonlinear system theory, 
which will serve as a foundation for the convergence analysis 
and stochastic approximation methods developed later. 
Consider the nonlinear dynamical system
\begin{align}
\frac{d}{dt}x_t = f(x_t), 
\quad x_0 \in {\mathbb R}^n, 
\quad t \ge 0, 
\label{eq:nonlinear-system}
\end{align}
where $x_t \in {\mathbb R}^n$ denotes the state, 
$t \ge 0$ is the time variable, 
$x_0 \in {\mathbb R}^n$ is the initial condition, 
and $f : {\mathbb R}^n \to {\mathbb R}^n$ is a nonlinear mapping. 
For simplicity, we assume that the solution to~\eqref{eq:nonlinear-system} 
exists and is unique, which holds whenever $f$ is globally Lipschitz continuous.

\begin{lemma}{\cite{khalil2002nonlinear}}\label{lemma:existence}
Consider the nonlinear system~\eqref{eq:nonlinear-system}, 
and suppose that $f$ is globally Lipschitz continuous, 
i.e., $\|f(x)-f(y)\|\le l \|x-y\|$ for all $x,y \in {\mathbb R}^n$ 
and some constant $l>0$ with respect to a norm $\|\cdot\|$. 
Then, for any $x_0 \in {\mathbb R}^n$, the system admits a unique solution $x_t$ for all $t \ge 0$.
\end{lemma}

An equilibrium point is a fundamental notion in nonlinear system theory. 
A point $x_\infty \in {\mathbb R}^n$ is called an equilibrium of~\eqref{eq:nonlinear-system} 
if the state remains constant at $x_\infty$ whenever $x_0 = x_\infty$~\citep{khalil2002nonlinear}. 
Equivalently, equilibrium points correspond to the real roots of $f(x)=0$. 
An equilibrium $x_\infty$ is said to be \emph{globally asymptotically stable} 
if $x_t \to x_\infty$ as $t \to \infty$ for all initial states $x_0 \in {\mathbb R}^n$.

\subsection{ODE-based stochastic approximation}\label{sec:ODE-stochastic-approximation}
Owing to its broad applicability, the convergence of many reinforcement learning (RL) algorithms 
is commonly analyzed using the ODE method~\citep{s2013stochastic,kushner2003stochastic}. 
This approach investigates the stability of an associated limiting ODE to understand the behavior of stochastic recursions with diminishing step sizes, 
which asymptotically follow the trajectory of the ODE. 
Among such techniques, one of the most influential is the ODE method proposed by~\citet{borkar2000ode}. 
In the following, we provide a brief overview of their ODE-based convergence analysis framework 
for a general stochastic recursion of the form
\begin{align}
\theta_{k+1} = \theta_k + \alpha_k \bigl(f(\theta_k) + \varepsilon_{k+1}\bigr),
\label{eq:general-stochastic-recursion}
\end{align}
where $f:{\mathbb R}^n \to {\mathbb R}^n$ is a nonlinear mapping.

The Borkar–Meyn theorem establishes that under \Cref{assumption:1}, 
the stochastic process $(\theta_k)_{k=0}^\infty$ generated by~\eqref{eq:general-stochastic-recursion} 
is almost surely bounded and converges to the equilibrium point $\theta_\infty$. 
This result forms the theoretical basis for proving the convergence of the proposed R-GTD algorithm.
The following technical assumptions are standard in stochastic approximation theory:

\begin{assumption}\label{assumption:1}
$\,$
\begin{enumerate}
\item The mapping $f:{\mathbb R}^n \to {\mathbb R}^n$ is globally Lipschitz continuous, 
and there exists a function $f_\infty:{\mathbb R}^n \to {\mathbb R}^n$ such that 
$\lim_{c\to\infty} \frac{f(cx)}{c} = f_\infty(x)$ for all $x \in {\mathbb R}^n$.

\item The origin is an asymptotically stable equilibrium of the ODE $\dot{\theta}_t = f_\infty(\theta_t)$.

\item The ODE $\dot{\theta}_t = f(\theta_t)$ admits a unique globally asymptotically stable equilibrium 
$\theta_\infty \in {\mathbb R}^n$, i.e., $\theta_t \to \theta_\infty$ as $t \to \infty$.

\item The sequence $\{\varepsilon_k, {\mathcal G}_k, k \ge 1\}$, 
with ${\mathcal G}_k = \sigma(\theta_i, \varepsilon_i, i \le k)$, 
is a Martingale difference sequence. 
Furthermore, there exists a constant $C_0 < \infty$ such that 
for any initial $\theta_0 \in {\mathbb R}^n$,
\[
{\mathbb E}\bigl[\|\varepsilon_{k+1}\|_2^2 \mid {\mathcal G}_k\bigr] 
\le C_0 \bigl(1 + \|\theta_k\|_2^2\bigr), \quad \forall k \ge 0,
\]
where $\|\cdot\|_2$ denotes the Euclidean norm.

\item The step-size sequence $\{\alpha_k\}$ satisfies
\begin{align}
\alpha_k > 0, \quad 
\sum_{k=0}^\infty \alpha_k = \infty, \quad 
\sum_{k=0}^\infty \alpha_k^2 < \infty.
\label{eq:step-size-rule}
\end{align}
\end{enumerate}
\end{assumption}

\begin{lemma}[Borkar and Meyn theorem, {\cite{borkar2000ode}}]\label{lemma:Borkar}
Suppose that \Cref{assumption:1} holds. Then, the following statements are true:
\begin{enumerate}
\item For any initial condition $\theta_0 \in {\mathbb R}^n$, 
the iterates remain bounded, i.e., $\sup_{k \ge 0} \|\theta_k\|_2 < \infty$ with probability one.

\item Moreover, $\theta_k \to \theta_\infty$ as $k \to \infty$ with probability one.
\end{enumerate}
\end{lemma}
The preceding assumptions, together with the Borkar–Meyn theorem,
provide the ODE-based foundation needed to analyze the stochastic
approximation aspects of the R-GTD algorithm.
In addition, to ensure that the limiting ODE admits a well-defined equilibrium, we first establish the existence of a saddle point for the associated min--max formulation. Beyond existence, our convergence analysis further requires understanding the stability properties of the continuous-time PDGD induced by this saddle-point structure. The following lemmas summarize key results on the existence of saddle points and the exponential stability of the corresponding PDGD:


\begin{lemma}{\cite{rockafellar2015convex}}
\label{lem:saddle-existence-simple}
Let $L(\theta,\lambda)$ be a function that is convex in $\theta$ and concave in
$\lambda$. Assume that
\begin{itemize}
    \item for every $\lambda$, the function $L(\cdot,\lambda)$ admits a minimizer,
    and
    \item the dual function $g(\lambda) := \min_{\theta} L(\theta,\lambda)$ is
    upper bounded and attains its maximum.
\end{itemize}
Then the saddle value exists and satisfies
\[
\min_{\theta} \max_{\lambda} L(\theta,\lambda)
=
\max_{\lambda} \min_{\theta} L(\theta,\lambda).
\]
\end{lemma}

\begin{lemma}{ \cite{qu2018exponential}}\label{lemma:1}
Consider the equality constrained optimization in~\eqref{eq:constrained-optimization}, and suppose that $f$ is twice differentiable, $\mu$-strongly convex, and $l$-smooth, i.e, for all $x,y \in {\mathbb R}^n$, $\mu \left\| {x - y} \right\|_2^2  \le (\nabla f(x) - \nabla f(y))^\top (x-y) \le l\left\| {x - y} \right\|_2^2$. Moreover, suppose that $A$ is full row rank. Consider the corresponding Lagrangian function~\eqref{eq:constrained-optimization-lagrangian}. Then, the corresponding saddle-point $(\theta^*,\lambda^*)$ is unique, and the corresponding continuous-time PDGD,
\begin{align*}
\dot x_t =  - \nabla _\theta  L(\theta_t ,\lambda_t ) =  - \nabla _\theta  f(\theta_t ) - A^\top \lambda_t,\quad
\dot \lambda_t = \nabla _\lambda  L(\theta_t ,\lambda_t ) = A\theta_t  - b,
\end{align*}
exponentially converges to $(\theta^*,\lambda^*)$.
\end{lemma}

Beyond the stability properties of the primal--dual dynamics, our convergence
analysis of the R-GTD algorithm toward the true value function $V^\pi$
also relies on a structural property of a key matrix arising in the analysis.

Recall that
\[
M = \Phi^\top D^\beta (\gamma P^\pi - I)\Phi,\quad
B = \Phi^\top D^\beta \Phi,\quad
b = \Phi^\top D^\beta R^\pi,
\]
where $B \succ 0$, and define $G := M^\top B^{-1} M$ which is positive semidefinite.

If $G$ is a singular matrix, a key step in our analysis is based on the following fact: the regularized matrix defined as
\[
G + c^{-1} H,
\]
with 
\begin{equation}\label{eqn:vv_perp}
H := \sum_{i=1}^k v_i v_i^\top,
\end{equation}
where $\{v_1,\dots,v_k\}$ forms an orthonormal basis of the nullspace of $G$,
remains nonsingular for every $c>0$.  
The next lemma establishes this fact by showing that adding a positive
regularization along the nullspace directions of a singular PSD matrix
yields a strictly positive definite matrix. 
This property will be crucial in the singular-case convergence analysis 
presented in \Cref{sec:convergence-towards-true}. Note that the following lemma does not rely on any new technical insight.
It follows directly from standard spectral decomposition arguments for
symmetric matrices (see, e.g., \citealp{horn2012matrix}), and is stated here for completeness and later use.

\begin{lemma}\label{lemma:nullspace-lifting}\label{app:nullspace-liftinig}
Let $G \in \mathbb{R}^{d\times d}$ be a symmetric positive semidefinite matrix,
and suppose that $G$ is singular with
\[
\mathrm{Null}(G) = \mathrm{span}\{v_1,\dots,v_k\},
\]
where $\{v_1,\dots,v_k\}$ is an orthonormal basis of the nullspace.
Define 
\[
H := \sum_{i=1}^k v_i v_i^\top.
\]
Then, for any $\alpha>0$, the regularized matrix
\[
G_\alpha := G + \alpha H
\]
is positive definite, and hence nonsingular.
\end{lemma}

\begin{proof}
For any nonzero $x \in \mathbb{R}^d$,
\[
x^\top G_\alpha x
= x^\top G x + \alpha x^\top H x.
\]

If $x \notin \mathrm{Null}(G)$, then $x^\top G x > 0$ because $G\succeq 0$
and the zero eigenvalues occur only along the nullspace directions.
Thus, $x^\top G_\alpha x > 0$. If $x \in \mathrm{Null}(G)\setminus\{0\}$, then 
$x = \sum_{i=1}^k c_i v_i$ for some nonzero vector $c$, and
\[
x^\top H x
= \sum_{i=1}^k (v_i^\top x)^2
= \sum_{i=1}^k c_i^2
= \|c\|^2
> 0.
\]
Thus,
\[
x^\top G_\alpha x = \alpha \|c\|^2 > 0.
\]

Since $x^\top G_\alpha x > 0$ for all $x\neq 0$, the matrix $G_\alpha$ is 
positive definite and hence nonsingular.
\end{proof}

\Cref{lemma:nullspace-lifting} ensures that the regularized matrix
$G + c^{-1}H$ is well-defined and invertible even when $G$ is singular.
To derive explicit expressions and error bounds in the singular case,
however, it is necessary to further characterize the structure of this
regularized matrix. To this end, we require several auxiliary results from linear algebra.
In particular, the following lemmas describe the geometric decomposition
induced by a symmetric matrix, the structure of orthogonal projections,
and the interaction between the Moore--Penrose pseudoinverse and the
nullspace. These results will be combined to prove
\Cref{lem:regularized-inverse}, which plays a central role in the subsequent convergence analysis.
\begin{lemma}{\citep[p.~603]{nocedal2006numerical}}
\label{lem:orthogonal-decomposition-G}
Let $G \in \mathbb{R}^{d\times d}$ be a symmetric matrix.
Then the space admits the orthogonal decomposition
\[
\mathbb{R}^d
=
\mathrm{Range}(G)\;\oplus\;\mathrm{Null}(G),
\]
where $\oplus$ denotes the \emph{orthogonal direct sum}, i.e., the sum of two
mutually orthogonal subspaces.
\end{lemma}

\begin{lemma}{\citep[p.~82]{golub2013matrix}}
\label{lem:projection-complement}
Let $S \subset \mathbb{R}^d$ be a subspace and let $\Pi_S$ denote the orthogonal
projection onto $S$. Then the orthogonal projection onto $S^\perp$ satisfies
\[
\Pi_{S^\perp} = I - \Pi_S.
\]
\end{lemma}

\begin{lemma}{\citep[p.~649]{boyd2004convex}}
\label{lem:pseudoinverse-projection}
Let $G \in \mathbb{R}^{m\times n}$ have singular value decomposition
$G = U \Sigma V^\top$, and let $G^\dagger$ denote its Moore--Penrose
pseudoinverse.
Then the matrices
\[
GG^\dagger = UU^\top,
\qquad
G^\dagger G = VV^\top
\]
are the orthogonal projection matrices onto $\mathrm{Range}(G)$ and
$\mathrm{Range}(G^\top)$, respectively.
\end{lemma}

\begin{lemma}\label{lem:PG_equal_zero}

Let $G \in \mathbb{R}^{n \times n}$ be a symmetric matrix and $\Pi_{\mathcal{N}(G)}$  be the orthogonal projection onto $\operatorname{Null}(G)$. Then, the following identities hold:

$$\Pi_{\mathcal{N}(G)} G^\dagger = 0$$

\end{lemma}

\begin{proof}

By the properties of the Moore-Penrose inverse, the range of $G^\dagger$ is equal to the range of $G^\top$. Since $G$ is symmetric, we have $\operatorname{Range}(G^\dagger) = \operatorname{Range}(G)$ \citep[p.~424]{meyer2023matrix}. By \Cref{lem:orthogonal-decomposition-G}, the range and null space of a symmetric matrix are orthogonal, i.e., $\operatorname{Range}(G) \perp \operatorname{Null}(G)$. Since $\Pi_{\mathcal{N}(G)}$ is the orthogonal projection onto $\operatorname{Null}(G)$, it maps any vector in $\operatorname{Range}(G)$ to zero. Consequently, for any vector $x$, $G^\dagger x \in \operatorname{Range}(G)$, which implies $\Pi_{\mathcal{N}(G)}(G^\dagger x) = 0$. Thus, $\Pi_{\mathcal{N}(G)} G^\dagger = 0$. 

\end{proof}

\begin{lemma}\label{lem:regularized-inverse}
For any $c>0$, the regularized matrix $G + c^{-1} \Pi_{\mathcal{N}(G)}$ is invertible and
its inverse admits the explicit form
\[
(G + c^{-1} \Pi_{\mathcal{N}(G)})^{-1} = G^\dagger + c \Pi_{\mathcal{N}(G)} .
\]
\end{lemma}


\begin{proof}
By \Cref{lem:pseudoinverse-projection}, $GG^\dagger$ is the orthogonal
projection onto $\operatorname{Range}(G)$. Since $G$ is symmetric,
\Cref{lem:orthogonal-decomposition-G} implies that
$\mathbb{R}^d = \operatorname{Range}(G) \oplus \operatorname{Null}(G)$.
Therefore, by \Cref{lem:projection-complement},
\[
GG^\dagger = \Pi_{\cal{R}(G)} = I - \Pi_{\cal{N}(G)}.
\]

Moreover, since $\Pi_{\cal{N}(G)}$ projects onto the null space of $G$,
we have $G \Pi_{\cal{N}(G)} = 0$.
By \Cref{lem:PG_equal_zero}, it also holds that
$\Pi_{\cal{N}(G)} G^\dagger = 0$.
Consequently,
\begin{align*}
(G + c^{-1} \Pi_{\cal{N}(G)})(G^\dagger + c \Pi_{\cal{N}(G)})
&= GG^\dagger + c\,G\Pi_{\cal{N}(G)} 
 + c^{-1} \Pi_{\cal{N}(G)} G^\dagger + \Pi_{\cal{N}(G)} \\
&= (I - \Pi_{\cal{N}(G)}) + 0 + 0 + \Pi_{\cal{N}(G)} \\
&= I,
\end{align*}
which proves the claim.
\end{proof}
Observe that the matrix $H$ in \eqref{eqn:vv_perp} is precisely the
orthogonal projection onto $\operatorname{Null}(G)$.
Consequently, the regularization term $c^{-1} H$ acts only on the
nullspace of $G$, leaving its range component unchanged.

\section{Proof of \Cref{prop:1}}
\label{app:prop-1}
\begin{tcolorbox}[colback=white,colframe=black,boxrule=0.8pt]
\noindent\textbf{Proposition~4.1 (restated).}
Let $(\theta_{\mathrm{RGTD}}, \lambda_{\mathrm{RGTD}}, w_{\mathrm{RGTD}})$
denote the optimal solution to the min-max problem in \Cref{problem:rgtd-lagrangian}.
Then the optimal solution admits the following closed-form expressions:
\begin{align*}
\theta_{\mathrm{RGTD}}
&=
\underbrace{
-\bigl(\Phi^\top D^{\beta}(\gamma P^\pi - I)\Phi\bigr)^{-1}
\Phi^\top D^{\beta} R^\pi
}_{\theta_{GTD2}}
\underbrace{
-\bigl(\Phi^\top D^{\beta}(\gamma P^\pi - I)\Phi\bigr)^{-1}
\Phi^\top D^{\beta}\Phi \, w_{\mathrm{RGTD}}
}_{\text{error}},
\\
\lambda_{\mathrm{RGTD}}
&=
-\bigl(\Phi^\top D^{\beta}(\gamma P^\pi - I)\Phi\bigr)^{-\top}
\Phi^\top D^{\beta}\Phi \, \theta_{\mathrm{RGTD}},
\\
w_{\mathrm{RGTD}}
&=
-\frac{1}{c}\,\lambda_{\mathrm{RGTD}}.
\end{align*}
\end{tcolorbox}

The optimal solution $(\theta_{\mathrm{RGTD}}, w_{\mathrm{RGTD}}, \lambda_{\mathrm{RGTD}})$
of \Cref{problem:rgtd-lagrangian} is characterized by the first-order
stationarity conditions with respect to the primal variables $(\theta,w)$
and the dual variable $\lambda$. Taking gradients of the Lagrangian
\[
L(\theta , w, \lambda )
=
\frac{c}{2}w^\top \Phi^\top D^{\beta} \Phi w
+
\frac{1}{2}\theta^\top \Phi^\top D^{\beta} \Phi \theta
+
\lambda^\top \Phi^\top D^{\beta}
(R^\pi + \gamma P^\pi \Phi \theta - \Phi \theta + \Phi w)
\]
and setting them to zero yields
\begin{align}
\nabla_\theta L &= \Phi^\top D^{\beta} \Phi \theta
+ (\gamma P^\pi \Phi - \Phi)^\top D^{\beta} \Phi \lambda = 0,
\label{eq:proof-theta}\\
\nabla_w L &= c\, \Phi^\top D^{\beta} \Phi w
+ \Phi^\top D^{\beta} \Phi \lambda = 0,
\label{eq:proof-w}\\
\nabla_\lambda L &=
\Phi^\top D^{\beta}
(R^\pi + \gamma P^\pi \Phi \theta - \Phi \theta + \Phi w) = 0.
\label{eq:proof-lambda}
\end{align}

From~\eqref{eq:proof-w}, we directly obtain
\begin{equation}\label{eq:proof-w-sol}
w = -\frac{1}{c}\,\lambda .
\end{equation}

Substituting~\eqref{eq:proof-w-sol} into~\eqref{eq:proof-theta} gives
\[
\Phi^\top D^{\beta} \Phi \theta
=
(\Phi - \gamma P^\pi \Phi)^\top D^{\beta} \Phi \lambda .
\]
Since $\Phi^\top D^{\beta}(\gamma P^\pi - I)\Phi$ is nonsingular under
\Cref{assumption:2}, this equation can be solved for $\theta$,
yielding
\[
\theta
=
-\bigl(\Phi^\top D^{\beta}(\gamma P^\pi - I)\Phi\bigr)^{-1}
\Phi^\top D^{\beta} \Phi \, \lambda .
\]

Finally, substituting this expression together with~\eqref{eq:proof-w-sol}
into~\eqref{eq:proof-lambda} and rearranging terms gives
\[
\theta
=
-\bigl(\Phi^\top D^{\beta}(\gamma P^\pi - I)\Phi\bigr)^{-1}
\Phi^\top D^{\beta} R^\pi
-
\bigl(\Phi^\top D^{\beta}(\gamma P^\pi - I)\Phi\bigr)^{-1}
\Phi^\top D^{\beta} \Phi \, w ,
\]
which yields the expression for $\theta_{\mathrm{RGTD}}$ in
\eqref{eq:theta-rgtd}. The expressions for $\lambda_{\mathrm{RGTD}}$
and $w_{\mathrm{RGTD}}$ then follow directly from
\eqref{eq:proof-theta} and~\eqref{eq:proof-w-sol}, completing the proof.

\section{Proof of~\Cref{problem:6}}\label{app:3}

\begin{tcolorbox}[colback=white,colframe=black,boxrule=0.8pt]
\noindent\textbf{Problem~6 (restated).}
Solve for $\theta \in \mathbb{R}^q$ the optimization
\begin{align*}
\min_{\theta \in \mathbb{R}^q}\;
\frac{c}{2}\lVert \Pi_{\mathcal{R}(\Phi)}(R^\pi + \gamma P^\pi \Phi\theta - \Phi\theta)\rVert_{D^{\beta}}^2
+ \frac{1}{2}\lVert \Phi\theta\rVert_{D^{\beta}}^2.
\end{align*}
This problem is equivalent to \Cref{problem:4} in the sense that both
formulations admit the same unique optimal solution for $\theta$.
\end{tcolorbox}

We begin by restating the optimization problem in \Cref{problem:4} for clarity.
We aim to find $(\theta, w) \in \mathbb{R}^q$ that solve
\begin{align}
\min_{\theta,w \in \mathbb{R}^q}
\quad & \frac{c}{2} w^\top \Phi^\top D^{\beta} \Phi w
      + \frac{1}{2} \theta^\top \Phi^\top D^{\beta} \Phi \theta \nonumber\\
{\rm subject\ to} 
\quad & 0 = \Phi^\top D^{\beta} \left( R^\pi 
      + \gamma P^\pi \Phi \theta - \Phi \theta + \Phi w \right). 
      \label{eq:problem-5-restated}
\end{align}

We now analyze the constraint~\eqref{eq:problem-5-restated}. 
The constraint can be written as
\begin{align*}
0 = \Phi^\top D^{\beta} \left( R^\pi 
+ \gamma P^\pi \Phi \theta
- \Phi \theta 
+ \Phi w \right).
\end{align*}
Rearranging the terms yields
\begin{align} \label{eqn:phi_w}
\Phi^\top D^{\beta} \Phi w
= -\Phi^\top D^{\beta} 
\left( R^\pi 
+ \gamma P^\pi \Phi\theta 
- \Phi\theta \right).
\end{align}

Since $\Phi$ has full column rank, the matrix $\Phi^\top D^{\beta} \Phi$ is positive definite.
Therefore, it is invertible, and we can solve for $w$ as
\begin{align*}
w 
= -(\Phi^\top D^{\beta} \Phi)^{-1} 
   \Phi^\top D^{\beta} (R^\pi + \gamma P^\pi \Phi \theta - \Phi \theta).
\end{align*}

This implies:
\begin{align*}
\Phi w = -\Pi_{\mathcal{R}(\Phi)}(R^\pi+\gamma P^\pi \Phi\theta-\Phi\theta)
\end{align*}
where $\Pi_{\mathcal{R}(\Phi)}:=\Phi(\Phi^\top D ^\beta \Phi)^{-1}\Phi ^\top D^\beta$ is the projection matrix onto Range($\Phi$). 
Substituting this expression for 
$\Phi w$ back into the objective function of \Cref{problem:4} yields:
\begin{align*}
\frac{c}{2} \lVert \Pi_{\mathcal{R}(\Phi)}(R^\pi + \gamma P^\pi \Phi\theta - \Phi\theta) \rVert_{D^\beta}^2 + \frac{1}{2} \lVert \Phi \theta \rVert_{D^\beta}^2,
\end{align*}
which matches exactly the objective function of \Cref{problem:6}.
Therefore, minimizing \Cref{problem:4} with respect to $(\theta, w)$ under the given constraint is equivalent to directly minimizing \Cref{problem:6} with respect to $\theta$.
\section{Verification of the full row rank condition}\label{app-prop:4}
\begin{proposition}\label{prop:4}
Consider the trajectory $(\theta_t, w, \lambda_t)$ of the PDGD in~\eqref{eq:continuous-rgtd}, and let $(\theta_{\mathrm{RGTD}}, w_{\mathrm{RGTD}},  \lambda_{\mathrm{RGTD}})$ be the corresponding saddle point. Then, $(\theta_t, w_t, \lambda_t)\to (\theta_{\mathrm{RGTD}}, w_{\mathrm{RGTD}},  \lambda_{\mathrm{RGTD}})$ exponentially as $t\to \infty$.
\end{proposition}
\begin{proof}
Note that~\eqref{eq:problem-5} has a strongly convex, smooth, and twice differentiable objective function. Moreover, constraint term can be presented as $0 = \Phi^\top D^{\beta}R^\pi
  + 
  \begin{bmatrix}
    \gamma\,\Phi^\top D^{\beta} P^\pi\Phi - \Phi^\top D^{\beta} \Phi
    & \Phi^\top D^{\beta}\Phi
  \end{bmatrix}
  \begin{bmatrix}
    \theta \\[6pt]
    w
  \end{bmatrix}.$ 
which is full row rank by assumption. These conditions ensure that all assumptions of \Cref{lemma:1} hold. Hence, by \Cref{lemma:1}, the PDGD of~\Cref{problem:4}, given in \eqref{eq:continuous-rgtd}, is globally asymptotically stable, and converges to its equilibrium point $(\theta_{RGTD}, w_{RGTD},\lambda_{RGTD}).$
\end{proof}
\section{Proof of~\Cref{thm:convergence2}}\label{app:2}
\begin{tcolorbox}[colback=white,colframe=black,boxrule=0.8pt]
\noindent\textbf{\Cref{thm:convergence2} (restated).}
Let us consider~\Cref{alg:rgtd-alg}, and assume that the step-size satisfy~\eqref{eq:step-size-rule}.
Then, $(\theta_k, w_k, \lambda_k)\to ( \theta_{\mathrm{RGTD}}, w_{\mathrm{RGTD}}, \lambda_{\mathrm{RGTD}})$ as $k\to \infty$ with probability one.
\end{tcolorbox}

The algorithm \eqref{eq:discrete-rgtd} can be rewritten by
\begin{align*}
\left[ {\begin{array}{*{20}c}
   {\theta _{k + 1} }  \\
   {w _{k+1} } \\
   {\lambda _{k + 1} }  \\
\end{array}} \right] = \left[ {\begin{array}{*{20}c}
   {\theta _k }  \\
   {w _k} \\
   {\lambda _k }  \\
\end{array}} \right] + \alpha _k \left( {f\left( {\left[ {\begin{array}{*{20}c}
   {\theta _k }  \\
   {w _k} \\
   {\lambda _k }  \\
\end{array}} \right]} \right) + \varepsilon _{k + 1} } \right),
\end{align*}
where
\begin{align*}
f\left( 
\begin{bmatrix}
\theta \\ w \\ \lambda
\end{bmatrix}
\right) 
:={}&
A
\begin{bmatrix}
\theta \\ w \\ \lambda
\end{bmatrix}
+ b,
\end{align*}



{\small
\setlength{\arraycolsep}{1pt}

\begin{align}
A &=
\left[
\begin{array}{ccc}
  -\Phi ^\top D^\beta R^\pi & 0 &
  -\Phi ^\top (\gamma P^\pi - I) ^\top D^\beta \Phi \\[0.6em]
  0 & -c\,\Phi ^\top D^\beta \Phi & -\Phi ^\top D^\beta \Phi \\[0.6em]
  \Phi ^\top D^\beta (\gamma P^\pi - I)\Phi & \Phi ^\top D^\beta \Phi & 0
\end{array}
\right],
\\[0.5em]
b &=
\begin{bmatrix}
  0 \\
  0 \\
  \Phi^\top D^\beta R^\pi
\end{bmatrix}.
\end{align}

\setlength{\arraycolsep}{5pt}
}

and
\begin{align*}
\varepsilon _{k + 1}  =
\left[
\begin{array}{c}
  -\Phi ^\top e_{s_k}e_{s_k}^\top\Phi \theta_{k}-((\gamma e_{s_k}e_{s_k'}\Phi)^\top e_{s_k}e_{s_k}^\top\Phi-\Phi ^\top e_{s_k}e_{s_k}^\top\Phi )\lambda_k)
   \\
   -c \Phi ^\top e_{s_k}e_{s_k}^\top\Phi w_k - \Phi ^\top e_{s_k}e_{s_k}^\top\Phi\lambda_k
   \\
   \Phi ^\top e_{s_k } e_{s_k }^\top (e_{s_k } r(s_k ,a_k ,s_k ')
\end{array}
\right] - f\!\left(
\begin{bmatrix}
   \theta _k \\ 
   w_k         \\ 
   \lambda _k
\end{bmatrix}
\right).
\end{align*}
The proof is completed by examining all the statements in~\Cref{assumption:1}. To prove the first statement of~\Cref{assumption:1}, we have
\begin{align*}
&\mathop {\lim }\limits_{c \to \infty } f\left( {c\left[ {\begin{array}{*{20}c}
   \theta   \\
   w        \\
   \lambda   \\
\end{array}} \right]} \right)/c
=f_\infty  \left( {\left[ {\begin{array}{*{20}c}
   \theta   \\
   w    \\
   \lambda   \\
\end{array}} \right]} \right)
 =A\left[ {\begin{array}{*{20}c}
   \theta   \\
   w        \\
   \lambda   \\
\end{array}} \right].
\end{align*}
Moreover, since $f$ is affine in its argument, it is globally Lipschitz continuous.
The second statement of~\Cref{assumption:1}: The PDGD of~\Cref{problem:4} can be written as
\begin{align*}
&\frac{d}{{dt}}\left[ {\begin{array}{*{20}c}
   {\theta_t - \theta _\infty }  \\
   {w_t - w _\infty} \\
   {\lambda_t - \lambda _\infty }  \\
\end{array}} \right]
= \bar{A}\left[ {\begin{array}{*{20}c}
   {\theta_t - \theta _\infty }  \\
   {w _t - w _\infty} \\
   {\lambda_t - \lambda _\infty }  \\
\end{array}} \right],
\end{align*}
where $\bar{A}$ is as follows

\begin{align*}
\setlength{\arraycolsep}{-2pt} 
\bar{A} =
\left[
\begin{array}{@{\hskip 0pt}ccc@{\hskip 0pt}} 
  -\Phi ^\top D^\beta R^\pi & 0 &
  -(\gamma P^\pi - I) ^\top D^\beta \Phi \\[0.5em]
  0 & -c\,\Phi ^\top D^\beta \Phi & -\Phi ^\top D^\beta \Phi \\[0.5em]
  \Phi ^\top D^\beta (\gamma P^\pi - I)\Phi & \Phi ^\top D^\beta \Phi & 0
\end{array}
\right],
\setlength{\arraycolsep}{5pt} 
\end{align*}

by \Cref{lemma:1} and \Cref{prop:4} the origin is the globally asymptotically stable equilibrium point. Now, one can observe that it is identical to the ODE
\[
\frac{d}{{dt}}\left[ {\begin{array}{*{20}c}
   {\theta  - \theta _\infty }  \\
   {w - w _\infty} \\
   {\lambda  - \theta _\infty }  \\
\end{array}} \right] = f_\infty  \left( {\left[ {\begin{array}{*{20}c}
   {\theta_t - \theta _\infty }  \\
   {w_t - w_\infty}             \\
   {\lambda_t - \theta _\infty }  \\
\end{array}} \right]} \right).
\]
Therefore, its origin is the globally asymptotically stable equilibrium point.
The third statement of~\Cref{assumption:1}: The ODE, $\frac{d}{{dt}}\left[ {\begin{array}{*{20}c}
   \theta_t \\
   w_t      \\
   \lambda_t \\
\end{array}} \right] = f\left( {\left[ {\begin{array}{*{20}c}
   \theta_t \\
   w_t      \\
   \lambda_t \\
\end{array}} \right]} \right)$, is identical to the PDGD of~\Cref{problem:4}. Therefore, it admits a unique globally asymptotically stable equilibrium point by~\Cref{lemma:1}.
Next, we prove the remaining parts. Recall that the R-GTD update. Define the history $\mathcal{G}_k$ as 
\begin{align*}
\mathcal{G}_k := (\varepsilon_k, \varepsilon_{k-1}, \ldots, \varepsilon_1, 
\theta_k, \theta_{k-1}, \ldots, \theta_0,w_k, w_{k-1}, \ldots, w_0, 
\lambda_k, \lambda_{k-1}, \ldots, \lambda_0)
\end{align*}
and the process $(M_k)_{k=0}^\infty$ with $M_k:=\sum_{i=1}^k {\varepsilon_i}$. Then, we can prove that $(M_k)_{k=0}^\infty$ is Martingale. To do so, we first prove ${\mathbb E}[\varepsilon_{k+1}|{\mathcal G}_k]=0$ by

\begin{align*}
&{\mathbb E}[\varepsilon _{k + 1} |{\mathcal G}_k ] \\
=&\, {\mathbb E}\!\left[
\left.
\left[
\begin{array}{*{20}{c}}
 -\Phi ^\top e_{s_k}e_{s_k}^\top\Phi \theta_{k}-((\gamma e_{s_k}e_{s_k'}\Phi)^\top e_{s_k}e_{s_k}^\top\Phi-\Phi ^\top e_{s_k}e_{s_k}^\top\Phi )\lambda_k)
   \\
   -c \Phi ^\top e_{s_k}e_{s_k}^\top\Phi w_k - \Phi ^\top e_{s_k}e_{s_k}^\top\Phi\lambda_k
   \\
   \Phi ^\top e_{s_k } e_{s_k }^\top (e_{s_k } r(s_k ,a_k ,s_k ')
 \Bigr)
\end{array}
\right]
\right| {\mathcal G}_k
\right]
- {\mathbb E}\!\left[
\left. f\!\left(\!\begin{bmatrix} \theta_k \\ w_k \\ \lambda_k \end{bmatrix}\!\right)
\right|{\mathcal G}_k
\right] \\
=&\,
{\mathbb E}\!\left[
\left. f\!\left(\!\begin{bmatrix} \theta_k \\ w_k \\ \lambda_k \end{bmatrix}\!\right)
\right|{\mathcal G}_k
\right]
- {\mathbb E}\!\left[
\left. f\!\left(\!\begin{bmatrix} \theta_k \\ w_k \\ \lambda_k \end{bmatrix}\!\right)
\right|{\mathcal G}_k
\right] = 0,
\end{align*}
where the second equality is due to the i.i.d. assumption of samples. Using this identity, we have
\begin{align*}
{\mathbb E}[M_{k+1}|{\mathcal G}_k]=& {\mathbb E}\left[ \left. \sum_{i=1}^{k+1}{\varepsilon_i} \right|{\mathcal G}_k\right]={\mathbb E}[\varepsilon_{k+1}|{\mathcal G}_k]+{\mathbb E}\left[ \left. \sum_{i=1}^k {\varepsilon_i} \right|{\mathcal G}_k \right]\\
=&{\mathbb E}\left[\left.\sum_{i=1}^k{\varepsilon_i} \right|{\mathcal G}_k \right]=\sum_{i=1}^k {\varepsilon_i}=M_k.
\end{align*}
Therefore, $(M_k)_{k=0}^\infty$ is a Martingale sequence, and $\varepsilon_{k+1} = M_{k+1}-M_k$ is a Martingale difference. Moreover, it can be easily proved that the second statement of the fourth condition of~\Cref{assumption:1} is satisfied by algebraic calculations. Therefore, the fourth condition is met.

\section{Deriving the normal equation for the R-GTD optimality condition}\label{app-lem:5}
We begin by recalling the notational conventions used throughout the main text:
\[
M = \Phi^\top D^\beta (\gamma P^\pi - I)\Phi,\quad
B = \Phi^\top D^\beta \Phi,\quad
b = \Phi^\top D^\beta R^\pi,
\]
where $B \succ 0$, and define 
\( G := M^\top B^{-1}M \), 
\( K := G^{-1} B G^{-1} M^\top B^{-1} b \).

This appendix provides the identity that allows the optimality condition of the R-GTD 
problem to be rewritten in a linear normal-equation form.  
Before deriving closed-form expressions of the R-GTD solution and establishing
error bounds, it is necessary to show that the optimal parameter $\theta$ satisfies
\[
\bigl(M^\top B^{-1} M + \tfrac{1}{c}B\bigr)\theta
= - M^\top B^{-1} b.
\]
The lemma below formally establishes this relation, which serves as
the starting point for the analysis developed in the main text.



\begin{lemma}\label{lem:5}
Let $M$, $B$, and $b$ be defined as in the notation above. Then the optimal solution $\theta$ of \Cref{problem:6} satisfies the linear system
\[
\bigl(M ^\top B^{-1} M + \tfrac{1}{c} B \bigr)\theta = - M ^\top B^{-1} b.
\]
\end{lemma}

\begin{proof}
From \Cref{problem:6}, using the projection $\Pi_{\mathcal{R}(\Phi)} = \Phi B^{-1}\Phi ^\top D^\beta$ 
and the definitions 
$M = \Phi ^\top D^\beta (\gamma P^\pi - I)\Phi$,
$B = \Phi ^\top D^\beta \Phi$,
and $b = \Phi ^\top D^\beta R^\pi$, 
we can rewrite the objective as
\[
J(\theta)
= \frac{c}{2}\theta ^\top M ^\top B^{-1} M \theta
+ c\, \theta ^\top M ^\top B^{-1} b
+ \frac{1}{2}\theta ^\top B \theta
+ \text{const.}
\]
Since $B \succ 0$, the problem is strongly convex.
Taking the gradient with respect to $\theta$ yields
\[
\nabla_\theta J(\theta)
= c\, M ^\top B^{-1} M \theta + c\, M ^\top B^{-1} b + B\theta.
\]
Setting $\nabla_\theta J(\theta)=0$ gives
\[
(c\, M ^\top B^{-1} M + B)\theta = - c\, M ^\top B^{-1} b,
\]
and dividing both sides by $c$ leads to
\[
(M ^\top B^{-1} M + \tfrac{1}{c} B)\theta = - M ^\top B^{-1} b.
\]
\end{proof}

\section{Proof of~\Cref{lem:6}}\label{app-lem:6}
\begin{tcolorbox}[colback=white,colframe=black,boxrule=0.8pt]
\noindent\textbf{\Cref{lem:6} (restated).}
Let $M,B,b,G,K$ be defined as above, 
and let $\theta_{\mathrm{RGTD}}$ be given by \eqref{eq:rgtd-solution}. Let $\Pi_{\mathcal{N}(G)}
$ denote the orthogonal projection onto the null space of $G$.
If $M$ is singular, then there exists a constant
$
c_0 := \|B - \Pi_{\mathcal{N}(G)}\|_2
$
such that for all $c > c_0$, every
$\theta_{\mathrm{GTD2}}\in\Theta_{\mathrm{GTD2}}$
defined in \eqref{eq:gtd2-set} admits the expansion
\begin{equation*}
\theta_{\mathrm{RGTD}}
=
\bigl(\theta_{\mathrm{GTD2}}
      - \Pi_{\mathcal{N}(G)}(\theta_{\mathrm{GTD2}})
 \bigr)
+ O\!\left(\frac{1}{c}\right).
\end{equation*}
If $M$ is nonsingular, then there exists a constant
$
c_0 := \|G^{-1}B\|_2
$
such that for all $c > c_0$, the unique GTD2 solution
$\theta_{\mathrm{GTD2}}$ given in
\eqref{eq:gtd2-solution-nonsingular} admits the expansion
\begin{equation*}
\theta_{\mathrm{RGTD}}
= \theta_{\mathrm{GTD2}}
+ \frac{1}{c} K
+ O\!\left(\frac{1}{c^2}\right).
\end{equation*}
\end{tcolorbox}

We prove the two cases separately.

\noindent\textbf{(a) Nonsingular case.}  
Let $G := M ^\top B^{-1}M$.  
If $G$ is nonsingular, the R-GTD solution can be written as
\begin{equation*}
\theta_{RGTD} = -(G + \tfrac{1}{c}B)^{-1} M ^\top B^{-1} b.
\end{equation*}

To expand the inverse of \(G + \tfrac{1}{c}B\), we first rewrite the matrix
as a perturbation of \(G\):
\[
G + \tfrac{1}{c}B 
= G\Bigl(I + \tfrac{1}{c} G^{-1}B\Bigr).
\]
Since \(G \succ 0\) in the nonsingular case, the matrix \(G^{-1}B\) is well-defined.
Moreover, for sufficiently large \(c\), the operator norm satisfies
\[
\left\|\tfrac{1}{c}G^{-1}B\right\|_2 < 1 .
\]

Define $c_0 := \|G^{-1}B\|_2 $, then, for all $c > c_0$, we have $\|c^{-1}G^{-1}B\|_2 < 1$, and the Neumann series expansion is valid.

Thus, the inverse of \(I + X\) with 
\[
X := \tfrac{1}{c}G^{-1}B
\]
can be expanded using the classical Neumann series
(see \citet{horn2012matrix}):
\[
(I + X)^{-1}
= I - X + X^2 - X^3 + \cdots, \qquad \|X\|_2 < 1.
\]
Applying this expansion yields
\[
(G + \tfrac{1}{c}B)^{-1}
= \Bigl(I + \tfrac{1}{c}G^{-1}B\Bigr)^{-1} G^{-1}
= \left(I - \tfrac{1}{c}G^{-1}B 
+ \tfrac{1}{c^{2}}\, G^{-1} B G^{-1} B 
+\cdots\right) G^{-1}.
\]
Multiplying out gives the first-order perturbation expansion
\[
(G + \tfrac{1}{c}B)^{-1}
= G^{-1} - \tfrac{1}{c} G^{-1} B G^{-1}
+ \tfrac{1}{c^{2}}\, G^{-1} B G^{-1} B G^{-1}+\cdots
\]
Therefore,
\[
\theta_{\mathrm{RGTD}} 
= -\Bigl(G^{-1} - \tfrac{1}{c} G^{-1} B G^{-1} + O(1/c^2)\Bigr) M ^\top B^{-1} b.
\]
Now, let us define 
\[
\theta_{\mathrm{GTD2}} := -G^{-1} M^\top B^{-1} b ,
\]
which coincides with the unique GTD2 solution in the nonsingular case
(cf. \eqref{eq:gtd2-solution-nonsingular} with \(G = M^\top B^{-1}M\)) and also define 
$
K := G^{-1} B G^{-1} M ^\top B^{-1} b
$.
Then, we obtain
\[
\theta_{\mathrm{RGTD}} = \theta_{\mathrm{GTD2}} + \frac{1}{c} K + O\!\left(\tfrac{1}{c^2}\right).
\]

\noindent\textbf{(b) Singular case.} 
Suppose that \(M\) is singular, and hence 
\(G = M^\top B^{-1} M\) is positive semidefinite but not invertible.

Although \(G^{-1}\) does not exist, the regularized matrix
$
G + \tfrac{1}{c} B
$
is invertible for every \(c>0\) because \(B \succ 0\) implies  
\(G + \tfrac{1}{c} B \succeq \tfrac{1}{c}B \succ 0\).  
Therefore, the R-GTD solution
\[
\theta_{\mathrm{RGTD}}
= (G + c^{-1} B)^{-1} M^{\top} B^{-1} b
\]
is well-defined for all \(c>0\).

To relate this solution to the GTD2 solutions, we first decompose any 
\(\theta_{\mathrm{GTD2}} \in \Theta_{GTD2}\) into its nullspace and orthogonal components:
\[
\theta_{\mathrm{GTD2}} = v + v_\perp, 
\qquad 
v \in \operatorname{Null}(G), 
\qquad 
v_\perp \in \operatorname{Null}(G)^\perp .
\]
Because \(Gv = 0\), the GTD2 equation \(G\theta_{GTD2} = M^{T}B^{-1}b\) reduces to
\[
G v_\perp = M^{\top} B^{-1} b.
\]
To express \(v_\perp\) explicitly, we introduce a small regularization 
in the nullspace direction:
\[
\left[ G + c^{-1} \sum_{i=1}^n v_i v_i^{\top} \right] v_\perp
= M^{\top} B^{-1} b,
\]
where ${c^{ - 1}}\sum\limits_{i = 1}^n {{v_i}} v_i^\top{v_ \bot } = 0$. 
By~\Cref{app:nullspace-liftinig} of Appendix, $G + c^{-1} \sum_{i=1}^n v_i v_i^{T}$ is invertible, and thus, we can derive
\[
v_\perp 
= \left( G + c^{-1} \sum_{i=1}^n v_i v_i^{T} \right)^{-1} 
M^{T} B^{-1} b.
\]

Note that although the auxiliary parameter $c$ is introduced to regularize the
operator and ensure nonsingularity, the orthogonal component $v_\perp$ of the
GTD2 solution is in fact independent of $c$.
This can be seen by explicitly characterizing the inverse of the regularized
operator. 

Recall that $\{v_1,\dots,v_k\}$ is a basis of
$\operatorname{Null}(G)$.
Then the matrix
\[
\sum_{i=1}^k v_i v_i^\top
\]
coincides with the orthogonal projection onto $\operatorname{Null}(G)$.
Indeed, for any $x \in \mathbb{R}^d$,
\[
\sum_{i=1}^k v_i v_i^\top x
= \sum_{i=1}^k v_i \langle v_i, x \rangle
= \Pi_{\cal{N}(G)} x,
\]
which implies
\[
\sum_{i=1}^k v_i v_i^\top = \Pi_{\cal{N}(G)}.
\]

It follows from \Cref{lem:regularized-inverse} that the inverse of the regularized operator admits
the explicit form
\[
(G + c^{-1} \Pi_{\cal{N}(G)})^{-1} = G^\dagger + c \Pi_{\cal{N}(G)} , \qquad c > 0.
\]

Since the GTD2 equation $G\theta = M^\top B^{-1} b$ is solvable, we have
$M^\top B^{-1} b \in \mathrm{Range}(G)$, and hence
\[
\Pi_{\cal{N}(G)}\, M^\top B^{-1} b = 0 .
\]
Consequently,
\[
\bigl(G + c^{-1} \Pi_{\cal{N}(G)}\bigr)^{-1} M^\top B^{-1} b
=
\bigl(G^\dagger + c \Pi_{\cal{N}(G)}\bigr) M^\top B^{-1} b
=
G^\dagger M^\top B^{-1} b,
\]
which shows that the orthogonal component $v_\perp$ is independent of $c$.

Thus,
\[
\theta_{\mathrm{GTD2}} 
= v + v_\perp
= \left( G + c^{-1} \sum_{i=1}^n v_i v_i^{\top} \right)^{-1}
      M^{T} B^{-1} b + v.
\]


Then,
\[
\theta_{\mathrm{GTD2}} 
-\Pi_{\cal{N}(G)}(\theta_{\mathrm{GTD2}})
= 
\left( G + c^{-1} \sum_{i=1}^n v_i v_i^{\top} \right)^{-1}
M^{T} B^{-1} b.
\]

Next, we decompose the R-GTD solution as
\[
\theta_{\mathrm{RGTD}} = u + u_\perp, 
\qquad 
u \in \operatorname{Null}(G), 
\qquad 
u_\perp \in \operatorname{Null}(G)^\perp .
\]
From the definition of \(\theta_{\mathrm{RGTD}}\),
\[
u_\perp 
= (G + c^{-1} B)^{-1} M^{\top} B^{-1} b - u.
\]

To compare this expression with the GTD2 decomposition, we rewrite
\[
G + c^{-1}B
=
G 
+ c^{-1} \sum_{i=1}^n v_i v_i^{\top} 
+ c^{-1}\!\left( B - \sum_{i=1}^n v_i v_i^{\top} \right).
\]
Factoring out the first term gives
\[
(G + c^{-1}B)^{-1}
=
\left[ I + c^{-1}\!\left( B - \sum_{i=1}^n v_i v_i^{\top} \right) \right]^{-1}
\left( G + c^{-1} \sum_{i=1}^n v_i v_i^{\top} \right)^{-1}.
\]

Let $c_0 := \|B - \Pi_{\mathcal{N}(G)}\|_2$. For all $c > c_0$, the bracketed term is a small perturbation of the identity since $\|\frac{1}{c}(B - \sum v_i v_i^\top)\|_2 < 1$. Thus, its inverse admits the first-order expansion:
\[
\left[ I + c^{-1}\!\left( B - \sum_{i=1}^n v_i v_i^{\top} \right) \right]^{-1}
=
I + O(1/c).
\]
Substituting this into the expression for \(u_\perp\) yields
\[
u_\perp 
=
\left( G + c^{-1} \sum_{i=1}^n v_i v_i^{\top} \right)^{-1}
M^{\top} B^{-1} b
+ O\!\left(\tfrac{1}{c}\right) - u.
\]
Therefore,
\[
\theta_{\mathrm{RGTD}}
= u + u_\perp
=
\left( G + c^{-1} \sum_{i=1}^n v_i v_i^{\top} \right)^{-1}
M^{T} B^{-1} b
+ O\!\left(\tfrac{1}{c}\right).
\]

Using the earlier identity for the GTD2 orthogonal component,
\[
\theta_{\mathrm{GTD2}} - \Pi_{\cal{N}(G)}(\theta_{\mathrm{GTD2}})
=
\left( G + c^{-1} \sum_{i=1}^n v_i v_i^{\top} \right)^{-1}
M^{T} B^{-1} b,
\]
we conclude that
\[
\theta_{\mathrm{RGTD}}
=
\left(
\theta_{\mathrm{GTD2}} - 
\Pi_{\cal{N}(G)}(\theta_{\mathrm{GTD2}})
\right)
+ O\!\left(\tfrac{1}{c}\right).
\]

Finally, letting \(c \to \infty\) yields the limit
\[
\lim_{c \to \infty} \theta_{\mathrm{RGTD}}
=
\theta_{\mathrm{GTD2}} 
- \Pi_{\cal{N}(G)}(\theta_{\mathrm{GTD2}}),
\]
which shows that the R-GTD solution removes the nullspace component 
of the GTD2 solution as the regularization parameter grows.

\section{Auxiliary lemma for error bounds}\label{app-lem:7}


\begin{lemma}\label{lem:7}
Let $\operatorname{dist}(x, S) := \inf_{y \in S} \|x - y\|_2$
denote the Euclidean distance from a point $x$ to a convex set $S$. Then

\noindent\textbf{(a) Nonsingular case.} 
If $M$ is nonsingular, the GTD2 solution $\theta_{\mathrm{GTD2}}$ is exist. Define
$
\Gamma :=
\|\Phi\|_2
\|G^{-1} B G^{-1} B G^{-1}\|_2
\|M^\top B^{-1} b\|_2.
$
\begin{align*}
\|\Phi\theta_{\mathrm{RGTD}} - \Phi\theta_{\mathrm{GTD2}}\|_2
\le \frac{\|\Phi\|_2\,\|K\|_2}{c} +  O\!\left(\tfrac{\Gamma}{c^2}\right).
\end{align*}

\noindent\textbf{(b) Singular case.}
If $M$ is singular, the GTD2 solution set is
\begin{equation}
\label{app:distance-nonsingular}
\Theta_{\mathrm{GTD2}} = \theta_{\mathrm{GTD2}} + \mathrm{Null}(G),
\end{equation}
where $\theta_{\mathrm{GTD2}}$ denotes an arbitrary but fixed element of $\Theta_{\mathrm{GTD2}}$. Then
\begin{equation}
\label{app:distance-singular}
\mathrm{dist}\bigl(\Phi\theta_{\mathrm{RGTD}},\, \Phi\Theta_{\mathrm{GTD2}}\bigr)
= O\!\left(\tfrac{1}{c}\right).
\end{equation}
\end{lemma}

\begin{proof}
Recall that
\[
M = \Phi^\top D^\beta (\gamma P^\pi - I)\Phi,\quad
B = \Phi^\top D^\beta \Phi \succ 0,\quad
b = \Phi^\top D^\beta R^\pi,
\]
and define \( G := M^\top B^{-1} M \).
We consider the two cases separately.

\noindent\textbf{(a) Nonsingular case.}  
Let $G = M ^\top B^{-1}M$, which is nonsingular.  
From \Cref{lem:6}, the R-GTD solution admits the expansion
\begin{align*}
\theta_{\mathrm{RGTD}} = \theta_{\mathrm{GTD2}} + \frac{1}{c} K + O\!\left(\tfrac{1}{c^2}\right), \quad K = G^{-1} B G^{-1} M ^\top B^{-1} b.
\end{align*}
Multiplying both sides by $\Phi$, we get 
\[
\Phi\theta_{\mathrm{RGTD}} - \Phi\theta_{\mathrm{GTD2}}
= \frac{1}{c}\Phi K + O\!\left(\tfrac{1}{c^2}\right).
\]
Taking the norm, we have
\begin{align*}
\|\Phi\theta_{\mathrm{RGTD}} - \Phi\theta_{\mathrm{GTD2}}\|_2
&\le \frac{1}{c}\|\Phi K\|_2 + \|\Phi \cdot O(1/c^2)\|_2 \\
&\le \frac{\|\Phi\|_2\,\|K\|_2}{c} + O\!\left(\tfrac{\Gamma}{c^2}\right),
\end{align*}
where $\Gamma :=
\|\Phi\|_2
\|G^{-1} B G^{-1} B G^{-1}\|_2
\|M^\top B^{-1} b\|_2.$
To see this, recall the higher-order expansion of the inverse from
Lemma~\ref{lem:6}:
\[
(G + \tfrac{1}{c}B)^{-1}
= G^{-1}
- \tfrac{1}{c}G^{-1} B G^{-1}
+ \tfrac{1}{c^2}G^{-1} B G^{-1} B G^{-1}
+ \cdots .
\]
The term corresponding to $1/c^2$ in the expansion is
$
\frac{1}{c^2}
\bigl(G^{-1} B G^{-1} B G^{-1}\bigr) M^\top B^{-1} b.
$
By the definition of $\Gamma$, the norm of this second-order term
multiplied by $\Phi$ is bounded by $\Gamma/c^2$, which justifies the
expansion.

\noindent\textbf{(b) Singular case.}
When $M$ is singular, the GTD2 solution set is the affine space
\[
\Theta_{\mathrm{GTD2}}
=
\theta_{\mathrm{GTD2}} + \mathrm{Null}(G),
\qquad 
\theta_{\mathrm{GTD2}} \in \Theta_{\mathrm{GTD2}}.
\]
From the singular case analysis in \Cref{app-lem:6}, the R-GTD
solution satisfies
\[
\theta_{\mathrm{RGTD}}
=
\theta_{\mathrm{GTD2}} - \Pi_{\cal{N}(G)}(\theta_{\mathrm{GTD2}})
+ O\!\left(\tfrac{1}{c}\right),
\]
where the vector
\(
\theta_{\mathrm{GTD2}}^{\perp}
=
\theta_{\mathrm{GTD2}} - \Pi_{\mathcal{N}(G)}(\theta_{\mathrm{GTD2}})
\)
is the unique component of $\theta_{\mathrm{GTD2}}$ lying in
$\mathrm{Null}(G)^\perp$.
Since
\(
\theta_{\mathrm{GTD2}}^{\perp} \in \Theta_{\mathrm{GTD2}},
\)
by the definition of the distance to a set, we have
\[
\mathrm{dist}(\Phi\theta_{\mathrm{RGTD}},\Phi\Theta_{\mathrm{GTD2}})
\le
\|\Phi(\theta_{\mathrm{RGTD}}-\theta_{\mathrm{GTD2}}^{\perp})\|_2.
\]
Using the expansion above yields
\[
\theta_{\mathrm{RGTD}}-\theta_{\mathrm{GTD2}}^{\perp}
= O\!\left(\tfrac{1}{c}\right),
\]
and therefore
\[
\mathrm{dist}(\Phi\theta_{\mathrm{RGTD}},\Phi\Theta_{\mathrm{GTD2}})
\le
\|\Phi\|_2\, O\!\left(\tfrac{1}{c}\right)
=
O\!\left(\tfrac{1}{c}\right).
\]
\end{proof}

\section{Proof of~\Cref{thm:2}}\label{app-thm:2}
\begin{tcolorbox}[colback=white,colframe=black,boxrule=0.8pt]

\noindent\textbf{\Cref{thm:2} (restated).}
The prediction error satisfies
\begin{align*}
\|\Phi\theta_{\mathrm{RGTD}} - \Phi\theta_*^\pi\|_2
&\le 
\mathrm{dist}(\Phi\theta_*^\pi,\; \Phi\Theta_{\mathrm{GTD2}})
+
\bigl\|
\Pi_{\Phi\Theta_{\mathrm{GTD2}}}(\Phi\theta_{\mathrm{RGTD}})
-
\Pi_{\Phi\Theta_{\mathrm{GTD2}}}(\Phi\theta_*^\pi)
\bigr\|_2
\\&\quad+
O\!\left(\tfrac{1}{c}\right).
\end{align*}

If $M$ is nonsingular, then $\Theta_{\mathrm{GTD2}}=\{\theta_{\mathrm{GTD2}}\}$ 
and the projection operator becomes the identity.
In this case, the bound simplifies to
\begin{align*}
\|\Phi\theta_{\mathrm{RGTD}} - \Phi\theta_*^\pi\|_2
&\le
\|\Phi\theta_{\mathrm{GTD2}} - \Phi\theta_*^\pi\|_2
+ \frac{\|\Phi\|_2\|K\|_2}{c}
+ O\!\left(\tfrac{1}{c^2}\right).
\end{align*}
\end{tcolorbox}

We prove the two cases separately.

\noindent\textbf{(a) Nonsingular case.}  
From \Cref{lem:6}, the R-GTD solution admits the first-order expansion
\begin{align*}
\Phi\theta_{\mathrm{RGTD}}
&= \Phi\theta_{\mathrm{GTD2}} + \frac{1}{c}\Phi K + O\!\left(\tfrac{1}{c^2}\right),
\end{align*}
where
\[
K := G^{-1} B G^{-1} M ^\top B^{-1} b, \quad 
G := M ^\top B^{-1}M.
\]
Hence, subtracting $\Phi\theta_*^\pi$ and taking the 2-norm yields the asymptotic relation
\begin{align*}
\lVert\Phi\theta_{\mathrm{RGTD}} - \Phi\theta_*^\pi\rVert_2 
&= \lVert\Phi\theta_{\mathrm{GTD2}} - \Phi\theta_*^\pi 
+  \tfrac{1}{c}\Phi K 
+ O\!\left(\tfrac{1}{c^2}\right)\rVert_2 \le \lVert \Phi\theta_{\mathrm{GTD2}} - \Phi\theta_*^\pi \rVert_2 
+ \frac{\|\Phi\|_2\,\|K\|_2}{c} 
\\&\quad+ O\!\left(\tfrac{1}{c^2}\right).
\end{align*}
This proves the non singular case.

\noindent\textbf{(b) Singular case.}
When \(M\) is singular, the GTD2 solution set is the affine space
\(
\Theta_{\mathrm{GTD2}}
= \theta_{\mathrm{GTD2}} + \mathrm{Null}(G)
\), and the distance is presented as \eqref{app:distance-singular}. To relate \(\theta_{\mathrm{RGTD}}\) to the projected fixed point,
insert and subtract the projection onto \(\Phi\Theta_{\mathrm{GTD2}}\):
\begin{align*}
\|\Phi\theta_{\mathrm{RGTD}} - \Phi\theta_*^\pi\|_2
&\le 
\mathrm{dist}(\Phi\theta_{\mathrm{RGTD}}, \Phi\Theta_{\mathrm{GTD2}})
+
\mathrm{dist}(\Phi\theta_*^\pi, \Phi\Theta_{\mathrm{GTD2}})
\\&\quad+
\|\Pi_{\Phi\Theta_{\mathrm{GTD2}}}(\Phi\theta_{\mathrm{RGTD}})
    - \Pi_{\Phi\Theta_{\mathrm{GTD2}}}(\Phi\theta_*^\pi)\|_2.
\end{align*}
Using the estimate
\(
\mathrm{dist}(\Phi\theta_{\mathrm{RGTD}}, \Phi\Theta_{\mathrm{GTD2}})
= O(1/c)
\)
from above yields
\[
\|\Phi\theta_{\mathrm{RGTD}} - \Phi\theta_*^\pi\|_2
\le
\mathrm{dist}(\Phi\theta_*^\pi,\Phi\Theta_{\mathrm{GTD2}})
+ \|\Pi_{\Phi\Theta_{\mathrm{GTD2}}}(\Phi\theta_{\mathrm{RGTD}})
    - \Pi_{\Phi\Theta_{\mathrm{GTD2}}}(\Phi\theta_*^\pi)\|_2 + O\!\left(\tfrac{1}{c}\right),
\]
which establishes the singular case.

\newpage
\section{Additional experiments}\label{app:additional-exp}
This appendix provides additional experimental results that complement the
main findings by evaluating the behavior of R-GTD against several baseline algorithms across a broader
range of settings, including non-singular random MDPs, singular case,
canonical off-policy counterexamples, and stochastic tabular environments. In addition, we include a systematic robustness analysis via a condition number sweep, which examines the performance of the algorithms across a wide spectrum of ill-conditioned regimes.

All experiments, including those presented in the main paper, were conducted on a single workstation equipped with an Intel Core i7-11800H CPU, an NVIDIA GeForce RTX 3060 GPU, and 16GB RAM. Due to the lightweight nature of the environments, each experimental run required only a few minutes.

\subsection{Non-singular MDP}
We first examine convergence in the non-singular setting where \Cref{assumption:2} holds. The randomly generated MDPs with $100$ states, $10$ actions, and $\gamma = 0.9$ were used. Rewards $r(s,a,s')$ were drawn from a uniform distribution over $[-1,1]$ and then sparsified by setting elements with $|r(s,a,s')|\leq 0.2$ to zero. The feature matrix $\Phi$ was full column rank with 10 features, and regularization coefficients $c \in {0.2, 0.4, 1}$. For each MDP, the experiment was repeated $30$ times, and the distribution of $\|\theta_k - \theta_*^\pi\|_2$ across runs was summarized using box-and-whisker plots. \Cref{fig:nonsingular-case} shows that R-GTD exhibits convergence comparable to GTD2 in this setting. Moreover, larger $c$ accelerates convergence, consistent with theoretical predictions.



\begin{figure}[htbp]
    \centering
    \includegraphics[width=0.6\linewidth]{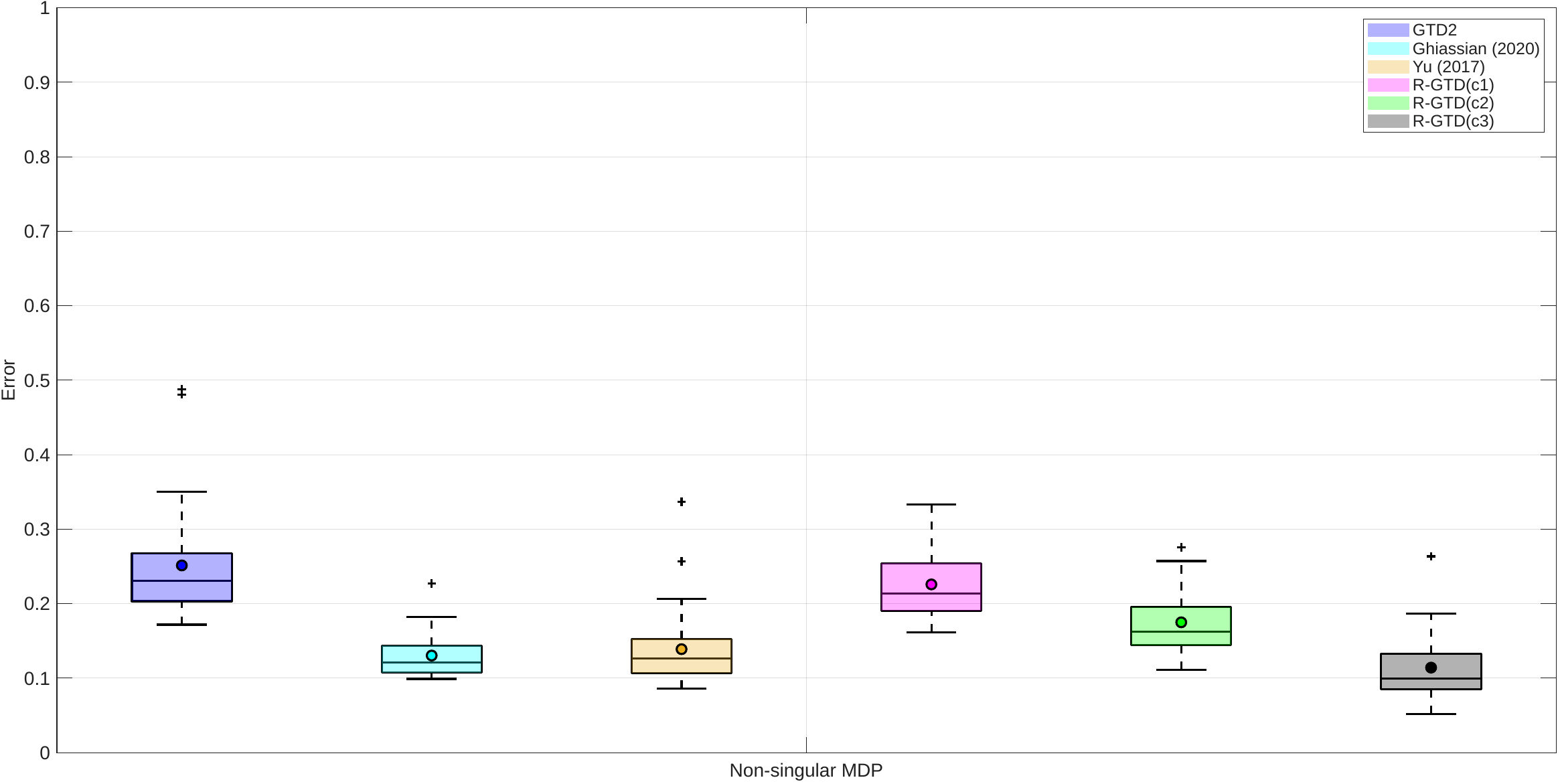}
    \caption{
        Box-and-whisker representation of $\|\theta_k - \theta_*^\pi\|_2$
        over 30 independent runs.
    }
    \label{fig:nonsingular-case}
\end{figure}

\begin{table}[h]
\centering
\caption{Final error of algorithms on non-singular MDP Environment}
\begin{tabular}{l|c}
\hline
Algorithm & Final Error ($mean \pm std$) \\ \hline
GTD2 & $0.2514 \pm 0.0768$ \\
Ghiassian (2020) & $0.1304 \pm 0.0299$ \\
Yu (2017) & $0.1390 \pm 0.0529$ \\
R-GTD(c1) & $0.2259 \pm 0.0445$ \\
R-GTD(c2) & $0.1750 \pm 0.0420$ \\
R-GTD(c3) & $\mathbf{0.1142 \pm 0.0441}$ \\
\hline
\end{tabular}
\end{table}

\newpage
\subsection{Robustness to initialization in the singular case}
\label{sec:robustness_initialization}
To further investigate the robustness of the proposed method in the singular regime, 
we evaluate its sensitivity to the magnitude of the initial null-space component. 
In particular, we consider different scaling factors for the initialization, 
which directly control the extent of the null-space error.

We conduct experiments with scaling factors of $100$, $200$, and $1000$. 
The case with scaling factor $1000$ is already presented in the main text, 
and here we provide additional quantitative results in terms of mean and variance 
to complement the visual comparison.

The experiments are conducted on a randomly generated MDP with $100$ states, 
$10$ actions, and discount factor $\gamma = 0.99$. Rewards $r(s,a,s')$ are 
drawn uniformly from $[-1,1]$ and sparsified by setting entries with 
$|r(s,a,s')|\le 0.2$ to zero. The feature matrix $\Phi$ has full column rank 
with $10$ features.

As shown in \Cref{tab:initialization_robustness}, R-GTD consistently achieves 
stable convergence across all initialization scales, while GTD2 exhibits 
significant variance. These results demonstrate that the stability of R-GTD 
is not due to favorable initialization, but rather an inherent property of 
the regularized formulation, which effectively suppresses the null-space 
component regardless of its magnitude.
\begin{figure}[htbp]
    \centering
    \begin{subfigure}[t]{0.48\linewidth}
        \centering
        \includegraphics[width=\linewidth]{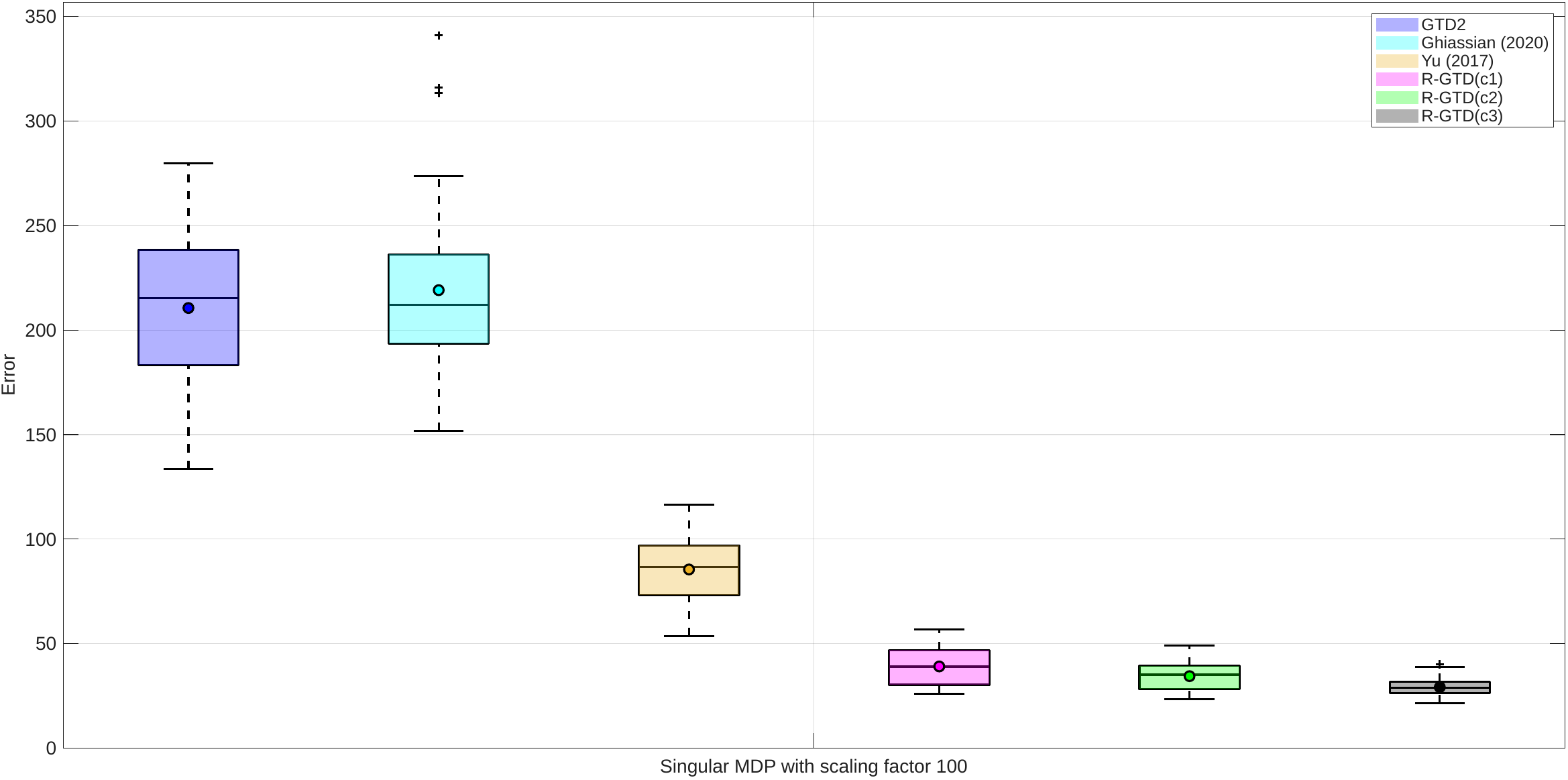}
        \caption*{scaling factor 100}
    \end{subfigure}
    \hfill
    \begin{subfigure}[t]{0.48\linewidth}
        \centering
        \includegraphics[width=\linewidth]{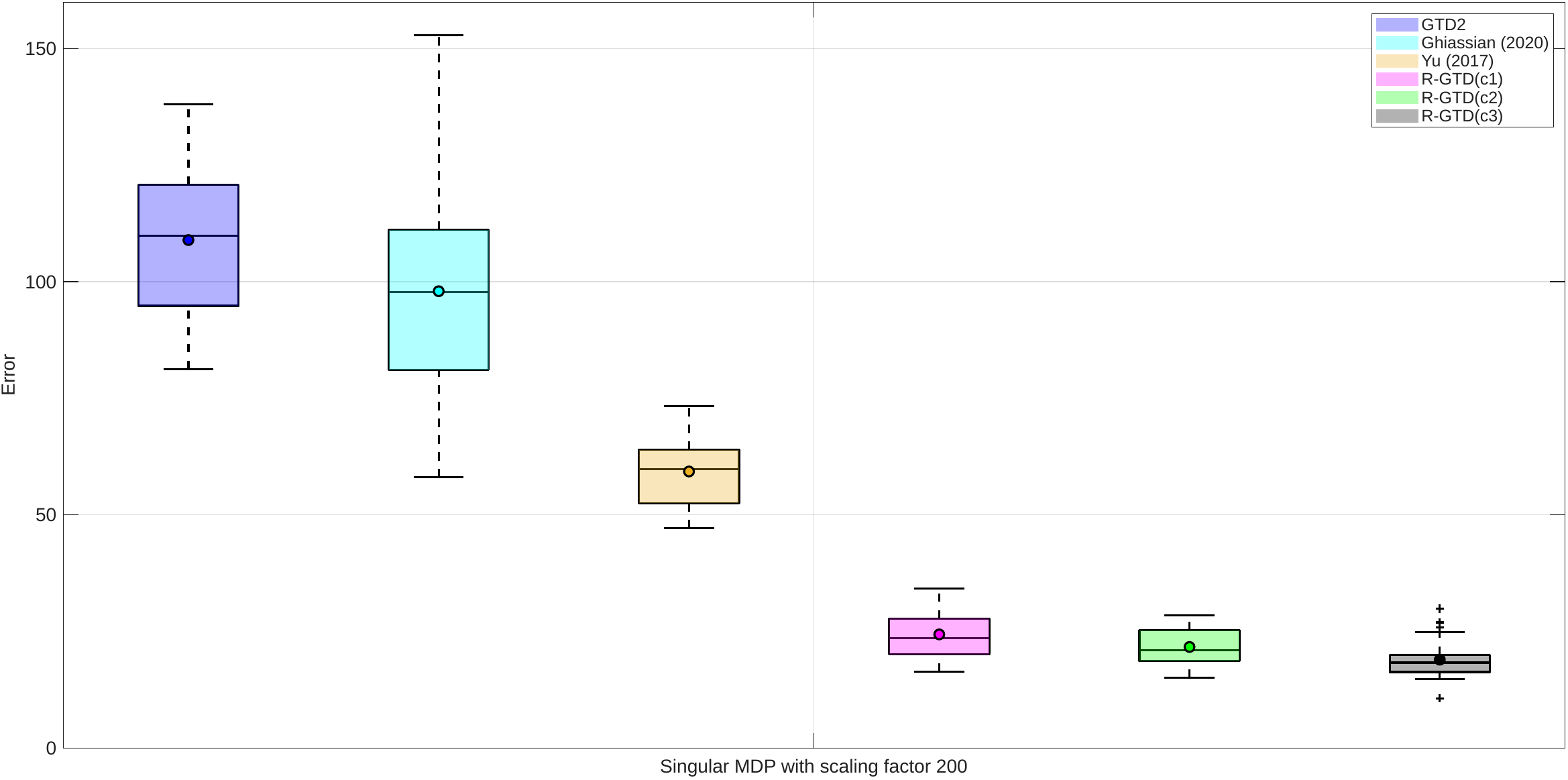}
        \caption*{scaling factor 200}
    \end{subfigure}
    
    \caption{
        Box-and-whisker representation of $\|\theta_k - \theta_*^\pi\|_2$
        over 30 independent runs.
    }
    \label{fig:singular_scaling_factor}
\end{figure}

\begin{table}[h]
\centering
\caption{Final error of algorithms under different magnitudes of null-space component initialization.}
\label{tab:initialization_robustness}
\begin{tabular}{l|c|c|c}
\hline
\multirow{2}{*}{Algorithm} & \multicolumn{3}{c}{Final Error ($\text{mean} \pm \text{std}$)} \\ \cline{2-4}
 & Scale = 100 & Scale = 200 & Scale = 1000 \\ \hline
GTD2 & $131.3479 \pm 12.9018$ & $108.9250 \pm 15.8008$ & $210.5517 \pm 33.0405$ \\
Ghiassian (2020) & $126.8904 \pm 17.3706$ & $97.9613 \pm 20.9245$ & $219.0772 \pm 45.9638$ \\
Yu (2017) & $49.4193 \pm 5.9066$ & $59.3214 \pm 7.3198$ & $85.5006 \pm 15.1741$ \\
R-GTD(c1) & $19.4368 \pm 5.0751$ & $24.3974 \pm 4.9124$ & $39.1171 \pm 9.3873$ \\
R-GTD(c2) & $17.5592 \pm 4.0874$ & $21.7006 \pm 4.0498$ & $34.4717 \pm 6.6838$ \\
R-GTD(c3) & $\mathbf{15.7314 \pm 3.2652}$ & $\mathbf{18.9793 \pm 3.9212}$ & $\mathbf{29.1966 \pm 4.1442}$ \\
\hline
\end{tabular}
\end{table}

\newpage
\subsection{Baird's counterexample}
We next evaluate off-policy stability using Baird’s counterexample
in~\citet{baird1995residual}, whose structure is illustrated in
\Cref{fig:baird} (Left). Quantitative results in \Cref{table:baird} reveal that 
R-GTD consistently achieves lower final error than GTD2 and Ghiassian (2020),
and performs competitively with Yu (2017). 

These results indicate that R-GTD not only ensures stable convergence 
in off-policy settings but also provides reliable and accurate 
solutions that are competitive with strong baseline algorithms.

\begin{figure}[htbp]
    \centering
    
    \begin{subfigure}[t]{0.48\linewidth}
        \centering
        \includegraphics[width=\linewidth]{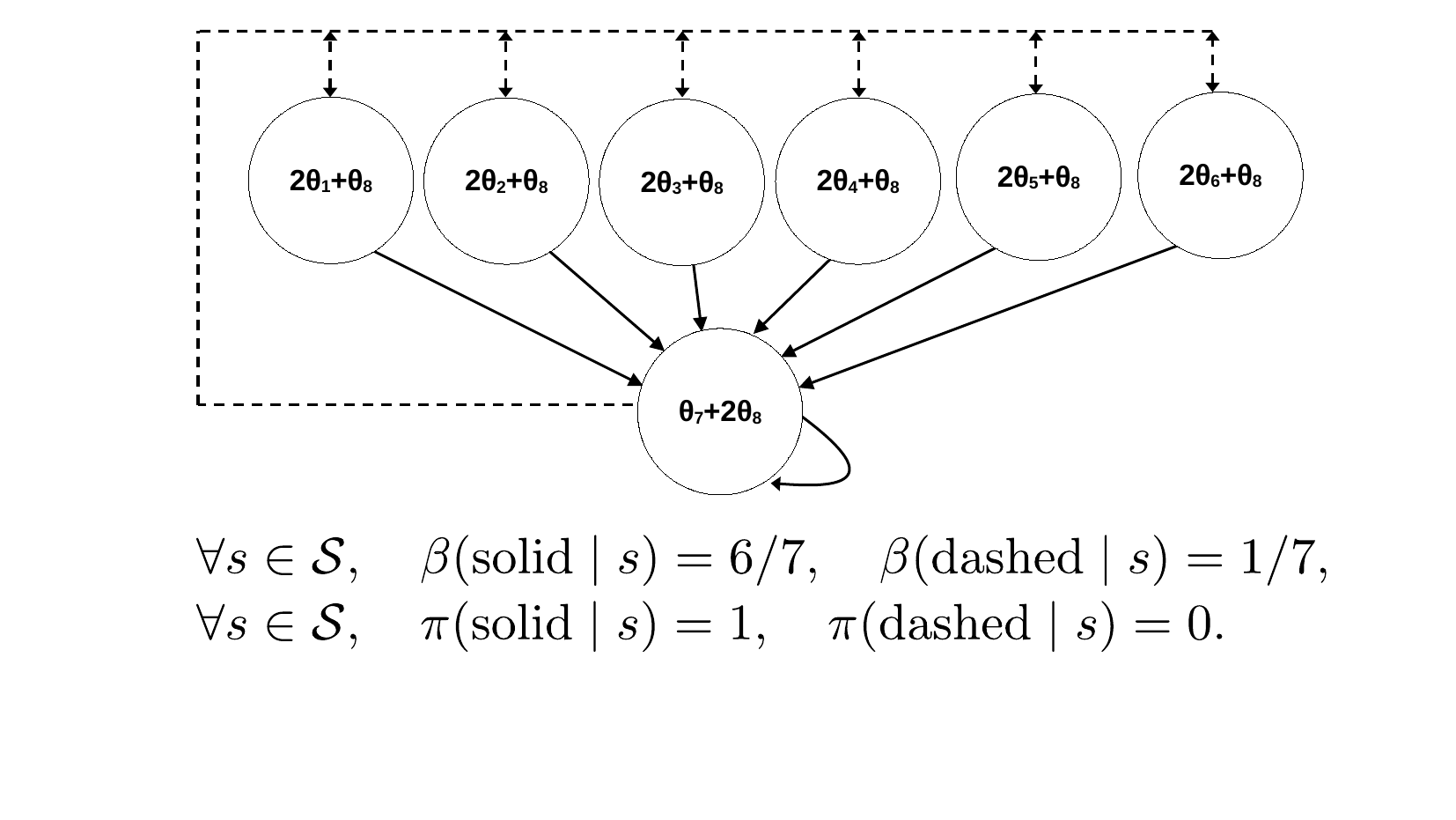}
        \caption{Baird's counterexample diagram}
    \end{subfigure}
    \hfill
    \begin{subfigure}[t]{0.48\linewidth}
        \centering
        \includegraphics[width=\linewidth]{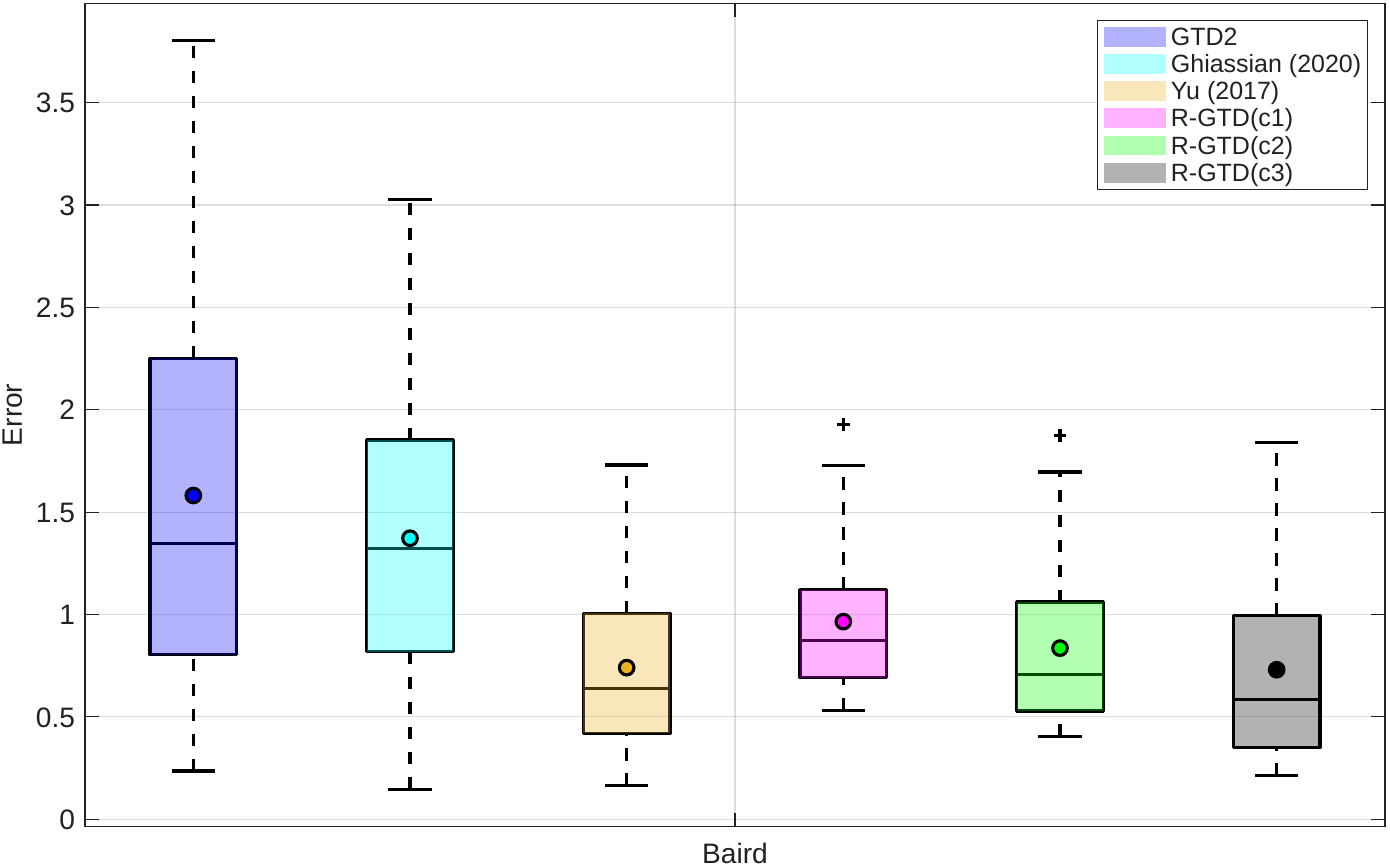}
        \caption{Performance comparison}
    \end{subfigure}
    
    \caption{
        Box-and-whisker representation of $\|\theta_k - \theta_*^\pi\|_2$
        over 30 independent runs.
    }
    \label{fig:baird}
\end{figure}

\begin{table}[h]
\centering
\caption{Final error of algorithms on Baird's counterexample.}
\label{table:baird}
\begin{tabular}{l|c}
\hline
Algorithm & Final Error ($mean \pm std$) \\ \hline
GTD2 & $1.5810 \pm 0.9377$ \\
Ghiassian (2020) & $1.3725 \pm 0.7256$ \\
Yu (2017) & $0.7403 \pm 0.4114$ \\
R-GTD(c1) & $0.9652 \pm 0.3710$ \\
R-GTD(c2) & $0.8357 \pm 0.4057$ \\
R-GTD(c3) & $\mathbf{0.7302 \pm 0.4531}$ \\
\hline
\end{tabular}
\end{table}

\newpage
\subsection{Frozen Lake}
We next evaluate the performance in a stochastic tabular environment, Frozen Lake~\citep{brockman2016openai}. 
The target policy $\pi$ is pre-trained via on-policy Q-learning~\citep{watkins1992q}, 
and the behavior policy is defined as an $\epsilon$-greedy variant of $\pi$ with $\epsilon=0.2$. 
Experiments were conducted using five different random seeds to ensure robustness.

Quantitative results in \Cref{table:frozen_lake} show that R-GTD consistently achieves lower final error 
compared to GTD2 and Ghiassian (2020), and performs competitively with Yu (2017). 
In particular, while Yu (2017) achieves the lowest mean final error, R-GTD with a larger 
regularization coefficient also attains favorable performance, demonstrating improved 
accuracy over several baselines across runs. These results indicate that R-GTD provides 
reliable and accurate value estimation even in stochastic environments with sparse rewards 
and transition uncertainty.
\begin{figure}[h]
\centering
\includegraphics[width=9cm,height=6cm]
{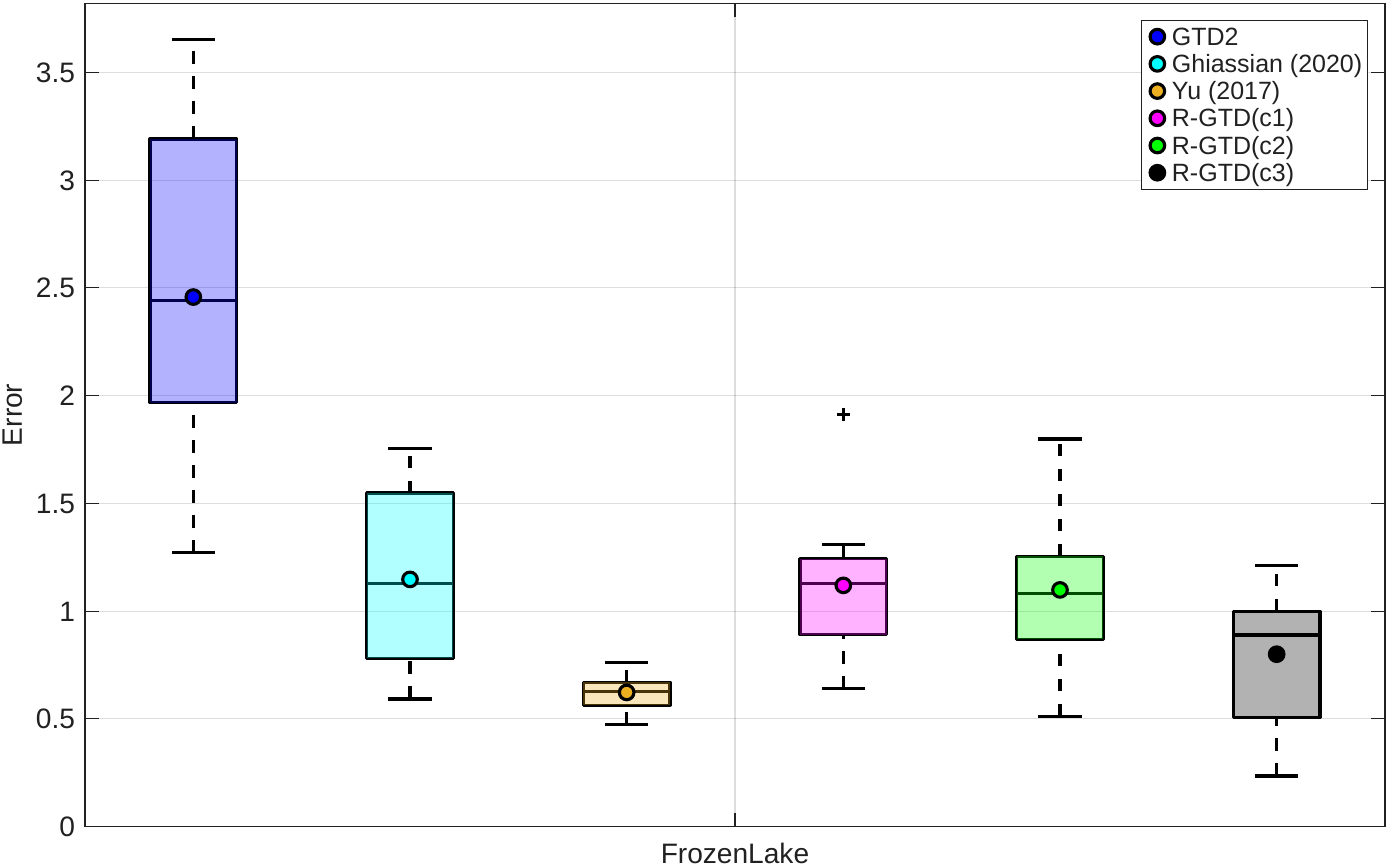}
\caption{Convergence of $\|\theta_k - \theta_*^\pi\|_2$ on the Frozen Lake environment with step-size $\alpha_k = 1/(k + 30)$ and regularization coefficients $c_1=0.2$, $c_2=0.4$, and $c_3=1$. 
}
\label{fig:frozenlake}
\end{figure}

\begin{table}[h]
\centering
\caption{Final error of algorithms on Frozen Lake environment.}
\label{table:frozen_lake}
\begin{tabular}{l|c}
\hline
Algorithm & Final Error ($mean \pm std$) \\ \hline
GTD2 & $2.4570 \pm 0.7865$ \\
Ghiassian (2020) & $1.1469 \pm 0.4201$ \\
Yu (2017) & $\mathbf{0.6224 \pm 0.0857}$ \\
R-GTD(c1) & $1.1191 \pm 0.3527$ \\
R-GTD(c2) & $1.0981 \pm 0.3739$ \\
R-GTD(c3) & $0.7997 \pm 0.3071$ \\
\hline
\end{tabular}
\end{table}

\newpage
\subsection{Robustness across varying condition numbers}
\label{sec:condition_number_sweep}
We also conducted a comprehensive condition number sweep analysis. This analysis systematically evaluates the robustness of R-GTD and various baselines under a wide spectrum of ill-conditioned regimes.

We systematically varied the condition number, allowing the feature interaction matrix (FIM) to smoothly transition from benign to highly ill-conditioned regimes. To rigorously evaluate the robustness of the algorithms, we adopted an initialization strategy that injects error specifically along the minimum singular vector. Each configuration was evaluated over 5 independent runs, with each run consisting of 100,000 iterations.

The empirical results are summarized in Tables below. The findings demonstrate two main conclusions.
First, R-GTD is consistently competitive with all evaluated baselines across the full spectrum of condition numbers. In particular, R-GTD achieves the lowest error and maintains stable variance.
Second, the performance advantage of R-GTD remains consistent as the condition number increases. This confirms that the stabilizing benefits of the proposed regularization are highly generalizable.


\begin{table*}[h]
\centering
\caption{Scale = 100, Final error of algorithms across varying condition numbers ($mean \pm std$)}
\label{tab:cond_robustness}
\resizebox{\textwidth}{!}{
\begin{tabular}{c|cccccc}
\hline
$\kappa(M)$ & GTD2 & Ghiassian (2020) & Yu (2017) & R-GTD(c1) & R-GTD(c2) & R-GTD(c3) \\ \hline
1.8e+04 & $93.38 \pm 21.95$ & $70.35 \pm 20.67$ & $56.27 \pm 10.66$ & $31.00 \pm 2.93$ & $29.23 \pm 2.34$ & $\mathbf{27.93 \pm 3.14}$ \\
8.5e+03 & $84.19 \pm 22.01$ & $86.81 \pm 16.40$ & $51.89 \pm 10.37$ & $32.53 \pm 8.32$ & $31.01 \pm 7.97$ & $\mathbf{29.79 \pm 7.53}$ \\
4.0e+03 & $142.38 \pm 21.23$ & $135.23 \pm 35.33$ & $53.21 \pm 8.71$ & $45.70 \pm 10.40$ & $44.18 \pm 9.01$ & $\mathbf{42.71 \pm 6.38}$ \\
1.8e+03 & $113.38 \pm 13.81$ & $108.65 \pm 10.68$ & $\mathbf{39.49 \pm 6.45}$ & $42.95 \pm 4.64$ & $41.26 \pm 3.34$ & $39.76 \pm 2.64$ \\
8.5e+02 & $148.44 \pm 25.93$ & $159.25 \pm 73.62$ & $57.13 \pm 12.20$ & $40.45 \pm 2.90$ & $39.34 \pm 3.06$ & $\mathbf{38.78 \pm 3.64}$ \\
4.0e+02 & $124.41 \pm 24.59$ & $125.65 \pm 20.37$ & $49.16 \pm 11.40$ & $38.65 \pm 1.82$ & $36.69 \pm 1.68$ & $\mathbf{34.96 \pm 2.99}$ \\
1.8e+02 & $95.75 \pm 27.73$ & $83.01 \pm 31.06$ & $51.76 \pm 13.99$ & $24.17 \pm 8.32$ & $22.36 \pm 8.17$ & $\mathbf{20.76 \pm 8.03}$ \\
8.5e+01 & $120.94 \pm 34.18$ & $116.86 \pm 25.05$ & $52.42 \pm 14.22$ & $27.81 \pm 5.42$ & $25.82 \pm 3.51$ & $\mathbf{24.27 \pm 2.44}$ \\
3.9e+01 & $101.46 \pm 15.93$ & $94.23 \pm 8.55$ & $44.90 \pm 7.08$ & $19.31 \pm 5.75$ & $17.85 \pm 4.06$ & $\mathbf{17.45 \pm 2.25}$ \\
1.8e+01 & $114.31 \pm 17.13$ & $73.41 \pm 18.11$ & $53.05 \pm 8.45$ & $21.31 \pm 5.74$ & $18.39 \pm 4.70$ & $\mathbf{15.11 \pm 4.00}$ \\
\hline
\end{tabular}
}
\end{table*}

\begin{table*}[h]
\centering
\caption{Scale = 200, Final error of algorithms across varying condition numbers ($mean \pm std$)}
\label{tab:cond_robustness}
\resizebox{\textwidth}{!}{
\begin{tabular}{c|cccccc}
\hline
$\kappa(M)$ & GTD2 & Ghiassian (2020) & Yu (2017) & R-GTD(c1) & R-GTD(c2) & R-GTD(c3) \\ \hline
1.8e+04 & $182.08 \pm 49.61$ & $145.66 \pm 61.35$ & $97.03 \pm 23.94$ & $39.45 \pm 6.25$ & $35.70 \pm 5.53$ & $\mathbf{34.14 \pm 5.00}$ \\
8.5e+03 & $205.04 \pm 43.69$ & $154.73 \pm 50.47$ & $107.58 \pm 19.84$ & $45.66 \pm 15.33$ & $38.68 \pm 12.25$ & $\mathbf{30.76 \pm 9.84}$ \\
4.0e+03 & $234.23 \pm 24.42$ & $212.84 \pm 29.62$ & $95.72 \pm 10.38$ & $61.13 \pm 8.28$ & $57.22 \pm 6.96$ & $\mathbf{53.84 \pm 5.39}$ \\
1.8e+03 & $248.99 \pm 35.18$ & $248.38 \pm 39.32$ & $102.79 \pm 16.63$ & $64.51 \pm 4.95$ & $59.31 \pm 4.41$ & $\mathbf{54.36 \pm 5.17}$ \\
8.5e+02 & $229.78 \pm 43.61$ & $241.77 \pm 62.74$ & $94.45 \pm 20.35$ & $62.02 \pm 11.65$ & $57.29 \pm 10.67$ & $\mathbf{51.22 \pm 8.44}$ \\
4.0e+02 & $253.45 \pm 52.43$ & $233.49 \pm 39.57$ & $109.10 \pm 22.21$ & $60.10 \pm 4.40$ & $56.71 \pm 5.33$ & $\mathbf{54.96 \pm 6.18}$ \\
1.8e+02 & $184.45 \pm 30.04$ & $182.95 \pm 51.00$ & $91.85 \pm 12.64$ & $39.27 \pm 14.66$ & $32.50 \pm 11.71$ & $\mathbf{26.32 \pm 9.11}$ \\
8.5e+01 & $213.12 \pm 38.16$ & $219.50 \pm 52.73$ & $93.35 \pm 17.91$ & $46.12 \pm 16.51$ & $41.94 \pm 14.06$ & $\mathbf{37.27 \pm 12.44}$ \\
3.9e+01 & $173.03 \pm 17.73$ & $145.80 \pm 35.54$ & $75.85 \pm 8.57$ & $43.63 \pm 6.78$ & $37.61 \pm 5.60$ & $\mathbf{31.01 \pm 5.05}$ \\
1.8e+01 & $198.77 \pm 24.50$ & $142.29 \pm 26.72$ & $92.11 \pm 11.49$ & $37.58 \pm 3.37$ & $32.93 \pm 4.70$ & $\mathbf{28.14 \pm 10.17}$ \\
\hline
\end{tabular}
}
\end{table*}

{\color{blue}
\begin{table*}[h]
\centering
\caption{Scale = 1000, Final error of algorithms across varying condition numbers ($mean \pm std$)}
\label{tab:cond_robustness}
\resizebox{\textwidth}{!}{
\begin{tabular}{c|cccccc}
\hline
$\kappa(M)$ & GTD2 & Ghiassian (2020) & Yu (2017) & R-GTD(c1) & R-GTD(c2) & R-GTD(c3) \\ \hline
1.8e+04 & $1181.67 \pm 282.87$ & $1074.19 \pm 578.73$ & $558.50 \pm 129.08$ & $203.37 \pm 66.93$ & $178.66 \pm 56.74$ & $\mathbf{154.03 \pm 50.56}$ \\
8.5e+03 & $1114.69 \pm 216.56$ & $933.68 \pm 191.36$ & $531.04 \pm 101.50$ & $199.60 \pm 78.71$ & $171.64 \pm 70.38$ & $\mathbf{149.06 \pm 68.87}$ \\
4.0e+03 & $1093.49 \pm 111.99$ & $959.53 \pm 197.04$ & $504.59 \pm 64.55$ & $237.01 \pm 76.21$ & $215.00 \pm 69.00$ & $\mathbf{199.50 \pm 68.96}$ \\
1.8e+03 & $1137.22 \pm 304.68$ & $1031.44 \pm 424.40$ & $507.83 \pm 141.80$ & $238.80 \pm 87.30$ & $208.44 \pm 75.97$ & $\mathbf{172.91 \pm 61.96}$ \\
8.5e+02 & $1133.01 \pm 266.62$ & $865.24 \pm 253.03$ & $521.64 \pm 116.49$ & $222.76 \pm 69.06$ & $206.70 \pm 53.29$ & $\mathbf{199.30 \pm 56.95}$ \\
4.0e+02 & $1159.39 \pm 226.08$ & $1151.81 \pm 284.92$ & $529.01 \pm 95.58$ & $263.20 \pm 89.91$ & $238.32 \pm 79.64$ & $\mathbf{218.26 \pm 65.28}$ \\
1.8e+02 & $1229.50 \pm 98.08$ & $1117.51 \pm 217.79$ & $574.89 \pm 39.56$ & $274.72 \pm 57.09$ & $235.77 \pm 35.80$ & $\mathbf{187.19 \pm 38.63}$ \\
8.5e+01 & $1093.71 \pm 232.09$ & $987.45 \pm 248.50$ & $507.03 \pm 114.99$ & $218.03 \pm 70.60$ & $199.07 \pm 55.70$ & $\mathbf{181.59 \pm 57.01}$ \\
3.9e+01 & $1170.70 \pm 262.15$ & $1011.09 \pm 230.20$ & $544.23 \pm 121.52$ & $205.55 \pm 50.54$ & $183.31 \pm 45.25$ & $\mathbf{172.34 \pm 42.13}$ \\
1.8e+01 & $1150.72 \pm 232.68$ & $1045.75 \pm 300.45$ & $539.49 \pm 110.71$ & $227.89 \pm 66.93$ & $207.26 \pm 63.80$ & $\mathbf{197.27 \pm 64.15}$ \\
\hline
\end{tabular}
}
\end{table*}
}

\newpage
\section{Explicit comparison with prior regularized methods}
\label{sec:comparison_tdrc}

To clarify the novelty of R-GTD, we compare it with prior regularized GTD methods.
The key distinction is not merely whether regularization is used, but \emph{where} it is applied and \emph{what geometry} it preserves in the singular feature-interaction regime.

\paragraph{1. TDRC: Regularization on the auxiliary variable.}
TDRC adds an \(L_2\) penalty to the secondary variable \(h\). In our notation, this gives
\[
    h_\beta = (B+\beta I)^{-1}(M\theta+b).
\]
The induced fixed-point condition for \(\theta\) becomes
\[
    (-M+\beta I)^\top(B+\beta I)^{-1}(M\theta+b)=0,
\]
with system matrix
\[
    (-M+\beta I)^\top(B+\beta I)^{-1}M.
\]
When \(M\) is singular, this induced system matrix can remain singular.
Thus, TDRC may improve stability through regularization of the auxiliary variable, but it does not directly guarantee uniqueness of the primary parameter \(\theta\).

\paragraph{2. Euclidean ridge methods: Regularization in parameter space.}
\citet{yu2017convergence} studies regularized GTD objectives of the form
\[
    J_p(\theta)=J(\theta)+p(\theta),
\]
where \(p(\theta)\) is a smooth convex regularizer. A common special case is the Euclidean ridge penalty
\[
    p(\theta)=\frac{\eta}{2}\|\theta\|_2^2.
\]
Similarly, \citet{zhang2020provably} use ridge regularization in their GQ2 critic.
For this Euclidean ridge choice, the induced regularized normal equation has the form
\[
    \left(M^\top B^{-1}M+\eta I\right)\theta
    =
    -M^\top B^{-1}b,
\]
This makes the surrogate problem well posed, but the regularization acts on the coordinate vector \(\theta\).
Hence, the selected solution is tied to the Euclidean parameter norm and can depend on the chosen feature parameterization.

This is different from analyzing the original singular GTD2 system.
In particular, these ridge-based formulations establish convergence or well-posedness of a modified objective, but they do not explicitly characterize how the regularized solution relates to the affine solution set of the unregularized GTD2 normal equation when \(M\) is singular. For instance, \citet{zhang2020provably} invokes a global nonsingularity-type condition, such as \(\inf_\theta |\det(A(\theta))|>0\), in its bias analysis, while
\citet{yu2017convergence} imposes an origin condition for unconstrained convergence.
These assumptions contrast with our direct geometric analysis of the strictly singular GTD2 normal equation.

\paragraph{3. R-GTD: Regularization in value-function space.}
R-GTD is derived by introducing a slack variable directly into the PBE constraint.
The resulting unconstrained form is
\[
    \min_{\theta\in\mathbb R^q}
    \frac{c}{2}
    \left\|
    \Pi_{\mathcal R(\Phi)}
    (R^\pi+\gamma P^\pi\Phi\theta-\Phi\theta)
    \right\|_{D^\beta}^2
    +
    \frac12\|\Phi\theta\|_{D^\beta}^2.
\]
Since
\[
    \|\Phi\theta\|_{D^\beta}^2=\theta^\top B\theta,
\]
the normal equation becomes
\[
    \left(M^\top B^{-1}M+\frac1cB\right)\theta
    =
    -M^\top B^{-1}b.
\]
Because \(B\succ 0\), the matrix
\[
    M^\top B^{-1}M+\frac1cB
\]
is strictly positive definite for every finite \(c>0\), even when \(M\) is singular.
Thus, R-GTD gives a unique solution by regularizing the represented value function \(\Phi\theta\), rather than the coordinate vector \(\theta\).

\paragraph{4. Solution-selection geometry in the singular regime.}
The difference between
\[
    M^\top B^{-1}M+\eta I
    \qquad \text{and} \qquad
    M^\top B^{-1}M+c^{-1}B
\]
is not merely a different choice of  regularization matrix.
It changes the geometry of the selected solution.

When \(M\) is singular, the unregularized GTD2 normal equation admits the affine solution set
\[
    \Theta_{\rm GTD2}
    =
    \{\theta : M^\top B^{-1}(M\theta+b)=0\}.
\]
Euclidean ridge methods select a solution according to the Euclidean norm of the parameter vector.
This yields a well-posed ridge surrogate, but it does not by itself explain which component of the original GTD2 solution set is identifiable from the projected Bellman equation.

By contrast, R-GTD regularizes the value function in the \(D^\beta\)-weighted feature space.
This allows us to explicitly characterize the singular limiting solution.
In particular, as \(c\to\infty\), our analysis shows that
\[
    \theta_{\rm RGTD}
    =
    \theta_{\rm GTD2}
    -
    \Pi_{\mathcal N(G)}(\theta_{\rm GTD2})
    +
    O(1/c),
    \qquad
    \theta_{\rm GTD2}\in\Theta_{\rm GTD2}.
\]
Thus, R-GTD removes the null-space component while preserving the component determined by the projected Bellman equation.
This solution-selection characterization is the main theoretical distinction from standard ridge-based regularized GTD methods, and it enables explicit error bounds in strictly singular settings.

\newpage
\section{Sensitivity analysis and practical guidelines for selecting $c$}
\label{sec:sensitivity_guideline}

As discussed in Section 4, the regularized matrix in R-GTD remains positive definite for any $c > 0$, ensuring a unique solution even when the feature interaction matrix (FIM) is singular. However, the choice of the regularization parameter $c$ introduces a fundamental bias-stability trade-off:

\begin{itemize}
    \item \textbf{Small $c$:} Strong regularization yields a highly well-conditioned and stable system, but it can introduce excessive bias, potentially degrading the empirical accuracy.
    \item \textbf{Intermediate $c$:} Provides the optimal trade-off between regularization-induced stability and bias reduction.
    \item \textbf{Large $c$:} Reduces the bias and allows the solution to approach that of GTD2, but it weakens the regularization effect, making the algorithm more susceptible to instability in ill-conditioned settings.
\end{itemize}

To provide a practical guideline for selecting $c$, we conducted a comprehensive sensitivity analysis across four different environments: a Nonsingular MDP, a Singular MDP, Baird's counterexample, and Frozen Lake. As shown in Table \ref{tab:sensitivity_c_combined}, intermediate values—specifically in the range of $c \in [1, 4]$—consistently achieve the best balance, yielding the lowest final errors across all tested scenarios. Based on these empirical results, we recommend starting with $c=1$ or $c=2$ as a robust default for practical applications, as it effectively stabilizes the learning process while keeping the residual bias minimal.

\begin{table}[h]
\centering
\caption{Sensitivity analysis of R-GTD across different values of the regularization parameter $c$. The table reports the Final Error ($\text{mean} \pm \text{std}$) across four diverse environments. Intermediate values ($c \in [1, 4]$) consistently strike the best balance between stability and bias.}
\label{tab:sensitivity_c_combined}
\vspace{2mm}
\resizebox{\textwidth}{!}{
\begin{tabular}{l|cccc}
\hline
\textbf{Algorithm} & \textbf{Nonsingular MDP} & \textbf{Singular MDP} & \textbf{Baird's Counterexample} & \textbf{Frozen Lake} \\ \hline
R-GTD (c=0.2) & $0.2254 \pm 0.0256$ & $43.8394 \pm 10.7775$ & $1.3156 \pm 0.5559$ & $1.1024 \pm 0.3450$ \\
R-GTD (c=0.4) & $0.1981 \pm 0.0217$ & $40.5704 \pm 9.0733$ & $1.2156 \pm 0.5742$ & $1.1257 \pm 0.3619$ \\
R-GTD (c=1)   & $0.1501 \pm 0.0193$ & $\textbf{37.1896} \pm \textbf{7.9214}$ & $1.1207 \pm 0.6186$ & $\textbf{0.8381} \pm \textbf{0.2262}$ \\
R-GTD (c=2)   & $0.1197 \pm 0.0237$ & $37.8707 \pm 8.9138$ & $\textbf{1.0982} \pm \textbf{0.6349}$ & $1.0654 \pm 0.4473$ \\
R-GTD (c=4)   & $\textbf{0.1073} \pm \textbf{0.0288}$ & $40.6424 \pm 10.4395$ & $1.1037 \pm 0.6285$ & $0.9186 \pm 0.3209$ \\
R-GTD (c=8)   & $0.1082 \pm 0.0316$ & $43.3760 \pm 11.7065$ & $1.1262 \pm 0.6130$ & $1.0455 \pm 0.4399$ \\
R-GTD (c=16)  & $0.1125 \pm 0.0329$ & $45.2561 \pm 12.5310$ & $1.1561 \pm 0.5979$ & $1.1277 \pm 0.2326$ \\ \hline
\end{tabular}
}
\end{table}

\newpage
\section{Singular case toy example}\label{app:toy}
\paragraph{Illustrative example (singular case).}
To illustrate the singular case, we consider a simple example 
where FIM is singular.
Specifically, we consider a simple Markov chain with three states and a single action.
We set the discount factor $\gamma=0.9$ and choose the state distribution matrix
$
D=\operatorname{diag}\!\Big(1/3,1/3,1/3\Big)$ where $\operatorname{diag}(\cdot)$ denotes a diagonal matrix with the given elements on its main diagonal.
The transition matrix, reward vector, and feature matrix are given by
\[
P=\begin{bmatrix}
\frac16 & \frac16 & \frac23\\
\frac16 & \frac16 & \frac23\\
\frac16 & \frac16 & \frac23
\end{bmatrix},
\quad
R=\begin{bmatrix}
1 \\
5 \\
5 
\end{bmatrix},
\quad
\Phi=\begin{bmatrix}
1 & 0\\
0 & 1\\
1 & 1
\end{bmatrix}.
\]
A direct computation shows that the matrix FIM is singular. 
As a consequence, the PBE admits an affine solution set
$\Theta_{\mathrm{GTD2}}$ defined in \eqref{eq:gtd2-set}. 
In particular, the solution set can be written as $
\Theta_{\mathrm{GTD2}} = \theta_{\mathrm{GTD2}} + \mathrm{Null}(G),
$ where $\theta_{\mathrm{GTD2}}$ is an arbitrary reference solution and
$\mathrm{Null}(G)$ denotes the null space of $G$. \Cref{fig:toy_examples} visualizes the closed-form trajectory of the
$\theta_{\mathrm{RGTD}}$, projected onto the value-function space
$\Phi\theta \in \mathbb{R}^2$, as the regularization parameter $c$ increases.
As $c$ grows, the trajectory of $\theta_{\mathrm{RGTD}}$ converges to the unique
GTD2 solution $\Pi_{\mathcal{N}(G)}(\theta_{\mathrm{GTD2}})$ lying in $\mathrm{Null}(G)^\perp$, as characterized in
\Cref{lem:6}.
This behavior matches the singular-case expansion in~\eqref{eq:rgtd-singular-expansion}.

\begin{figure}[h!]
  \centering
  \includegraphics[height=4.2cm]{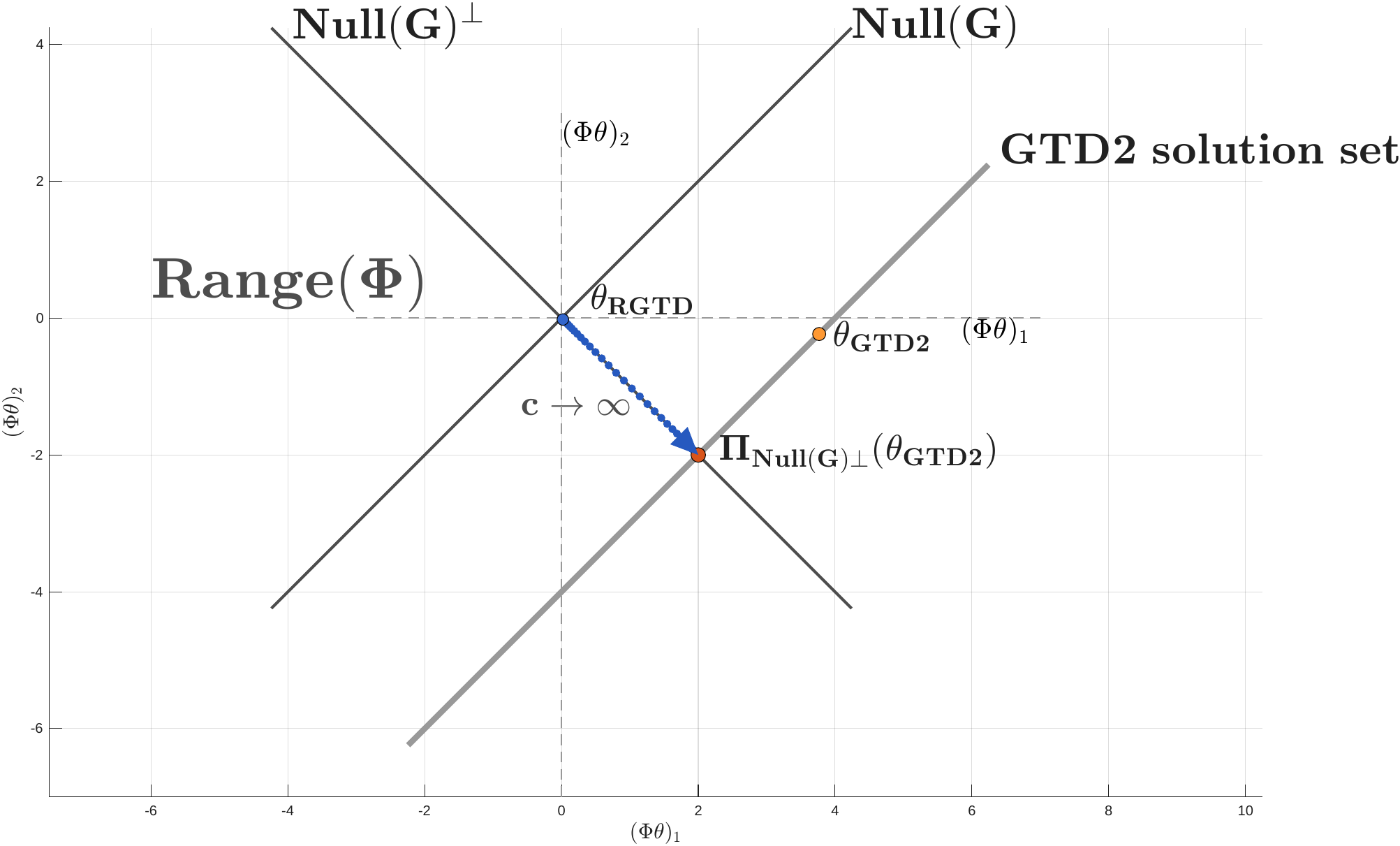}
  \caption{Solution trajectory illustrating the convergence of $\theta_{\mathrm{RGTD}}$ to $\theta_{\mathrm{GTD2}}$ as $c$ increases.}
  \label{fig:toy_examples}
\end{figure}

\newpage

\end{document}